\documentclass{article}

\usepackage[e]{esvect}
\usepackage{subcaption, comment} 
\excludecomment{mycomment}
\usepackage{multirow,multicol}

\usepackage{microtype}
\usepackage{graphicx}
\usepackage{booktabs} 

\usepackage{hyperref}

\usepackage[preprint]{icml2026}

\usepackage{amsmath}
\usepackage{amssymb}
\usepackage{mathtools}
\usepackage{amsthm}

\DeclarePairedDelimiter{\norm}{\lVert}{\rVert}

\usepackage[capitalize,noabbrev]{cleveref}

\theoremstyle{plain}
\newtheorem{theorem}{Theorem}[section]

\newtheorem{lemma}[theorem]{Lemma}
\newtheorem{corollary}[theorem]{Corollary}
\theoremstyle{definition}

\newtheorem{exmp}[theorem]{Example}
\newtheorem{condi}[theorem]{Condition} 
\theoremstyle{remark}
\newtheorem{remark}[theorem]{Remark}

\usepackage[textsize=tiny]{todonotes}

\icmltitlerunning{Distribution Modeling via Sparse Bayesian Networks}

\begin{document}

\twocolumn[

\icmltitle{Breaking the Curse with BAND: Nonparametric Distribution Estimation in High Dimensions}



\icmlsetsymbol{equal}{*}

\begin{icmlauthorlist}
\icmlauthor{Shuo-Chieh Huang}{ru}
\icmlauthor{Chien-Ming Chi}{as}
\icmlauthor{Jau-er Chen}{spe}
\end{icmlauthorlist}

\icmlaffiliation{ru}{Department of Statistics, Rutgers University, Piscataway, NJ, U.S.A.}
\icmlaffiliation{as}{Institute of Statistical Science, Academia Sinica, Taipei, Taiwan}
\icmlaffiliation{spe}{School of Political Science and Economics, National Taiwan University, Taipei, Taiwan}

\icmlcorrespondingauthor{Chien-Ming Chi}{xbbchi@stat.sinica.edu.tw}

\icmlkeywords{Sparse Bayesian networks; high-dimensional time series; discretization regression; forecasting confidence region; mixed-type data}



\vskip 0.3in
]



\printAffiliationsAndNotice{}  

\begin{abstract}

Minimax-optimal rates for multivariate distribution estimation are known to suffer from the curse of dimensionality. We propose a sparse Bayesian network approach in which each conditional probability is estimated using sparsity-aware conditional mean methods. The resulting estimator, \textit{BAyesian Network Distribution regression} (BAND), handles mixed data types in high-dimensional time series and achieves polynomial total variation convergence rates while allowing the feature dimension to grow polynomially with the sample size. These rates are substantially faster than the classical optimal rates for multivariate histogram density estimators that lack sparsity. Empirical evaluations show that BAND performs competitively for data sampling and confidence region forecasting against a range of state-of-the-art benchmarks.



\end{abstract}

\section{Introduction}

Distribution modeling underlies many statistical tasks, including multivariate confidence region construction \citep{vsidak1967rectangular}, out-of-distribution detection \citep{liu2024elephant}, missing value imputation \citep{Muzellec2020missing, huang2024temporal}, data visualization, and generative modeling. Despite its importance, nonparametric distribution estimation remains difficult in high dimensions. Classical theory shows that, without additional structure, the optimal rate of convergence in total variation is $n^{-\beta/(2\beta+p)}$, where $\beta > 0$ is related to the smoothness of the underlying density and $p$ the ambient dimension \citep{tsybakov2009introduction}, which is only informative when $p=o(\log n)$. This limitation reflects the curse of dimensionality.
Recent methods based on kernel estimation \citep{biroli2024kernel}, adaptive partitioning \citep{liu2023convergence}, distributional random forests \citep{cevid2022distributional}, and score-based diffusion models \citep{chen2022sampling, benton2023nearly, oko2023diffusion, zhang2024minimax} achieve minimax-optimal rates but still scale poorly with dimension. 

In contrast, by exploiting the underlying sparsity structure, high-dimensional conditional mean modeling can circumvent such dimensionality issue, using tools such as regularized regression \citep{tibshirani1996regression}, tree-based models \citep{breiman2017classification}, 
and neural networks with bounded fan-in \citep{lee1996efficient}. 
Inspired by the success in the aforementioned methods, we aim to tackle the curse of dimensionality in distribution modeling by combining a (sparse) Bayesian network representation and the state-of-the-art conditional mean modeling methods.


Using a Bayesian network~\citep{jordan1999introduction, bengio1999modeling}, the joint distribution of $p$ 
variables can be represented as a product of $p$ autoregressive conditional probability functions (see~\eqref{band} in Section~\ref{Sec2.2}). This representation underlies many modern architectures in language modeling~\citep{bengio2003neural, child2019generating}, image generation~\citep{van2016pixel, salimans2017pixelcnn++}, and genetic data analysis~\citep{friedman2000using}. 

In this paper, we propose a sparse Bayesian network model in which each conditional probability is estimated using conditional mean methods that exploit sparse dependence, such as $\ell_1$-regularized regression and tree-based models. The resulting estimator, \textit{BAyesian Network Distribution regression} (BAND), mitigates the curse of dimensionality in high-dimensional distribution modeling. 
The sparse Bayesian network assumption not only allows a formal bias-variance analysis of BAND's high-dimensional convergence rates, but is also flexible enough to capture complex structures such as mixture distributions.

Using one-hot encodings and discretization-based regression methods~\citep{foresi1995conditional, hall1999methods, chernozhukov2013inference, kneib2023rage}, BAND seamlessly handles discrete, continuous, and mixed data types. In addition, our theoretical results incorporate not only i.i.d.~data but also time series data in which case the stationary distribution is of interest. Under a sparse Bayesian network structure, where each variable depends on at most $s_0$ other variables, BAND achieves polynomial total variation convergence rates with the feature dimension allowed to grow polynomially with the sample size. For continuous time series with strong mixing, the convergence rate is of order $p\, n^{-1/(3+s_0)}$, up to mild subpolynomial factors.
This rate is comparable to the state-of-the-art result for high-dimensional density estimation \citep{vandermeulen2024breaking}.
For discrete distributions, we obtain a convergence rate which scales linearly with $p$ under the sparse Bayesian network assumption, which is far faster than the pessimistic minimax convergence rate without sparsity~\citep{Han2015minimax}.


We also evaluate BAND empirically on synthetic data, illustrating its performance in data sampling  and for forecasting confidence regions in monthly U.S. economic time series. In these experiments, BAND performs competitively with several generative models and density estimators, including normalizing flows~\citep{durkan2019neural} and high-dimensional copulas~\citep{nagler2016evading}. Implementations and experiments are available at \url{https://github.com/no-name213/band}.


\subsection{Related Work} 

Multivariate distribution modeling has accumulated a sizable literature for the fixed-dimensional regime. 
Copula-based methods often require restrictive assumptions, such as the simplified vine copula 
(\citealp{nagler2016evading}; see also Section~\ref{Sec4}), or are only supported by simulation studies \citep{oh2017modeling}. 
Kernel density estimation \citep{biroli2024kernel}, adaptive partitioning \citep{liu2023convergence}, and distributional random forests \citep{cevid2022distributional} are also classical methods that do not scale to high-dimensional time series data.

Recent literature also shed light on when the curse of dimensionality can be avoided using score-based diffusion methods \citep{song2020score}, such as Barron space-valued log-relative-density \citep{cole2024score}, low-dimensional support \citep{oko2023diffusion, beyler2025optimal}, and weak log-concavity \citep{silveri2025beyond}. Alternatively, \citet{liang2026denoising} characterizes the complexity of diffusion denoising through a notion of average curvature.

In comparison, BAND escapes the curse of dimensionality by leveraging the sparse Bayesian network that allows for flexible dependence structures. 
The sparse Bayesian network is similar to the graph resilience assumption of the recent work of \citet{vandermeulen2024breaking}, which was shown to determine the sample complexity for distribution estimation with i.i.d.~data supported on $[0,1]^p$. 
However, their work only proves the sample complexity whereas this paper develops an implementable method and is applicable to general time series with mixed-type data.

\subsection{Notation}


Let $(\Omega, \mathcal{F}, \mathbb{P})$ be a probability space. Boldface symbols such as $\vv{\boldsymbol{X}}$ denote random vectors, while $\vv{x}$ denotes deterministic vectors. The indicator function is $\boldsymbol{1}\{\cdot\}$, and $\mathcal{R}$ is the Borel $\sigma$-algebra on $\mathbb{R}$. For $\vv{x} \in \mathbb{R}^k$, let $\|\vv{x}\|_2 = (\sum_{j=1}^k x_j^2)^{1/2}$ and $\|\vv{x}\|_\infty = \max_{j} |x_j|$. We denote $[1\!:\!p] = \{1, \dots, p\}$ and write $\vv{h}_{1:p} = (\vv{h}_1, \dots, \vv{h}_p)$. For $S \subset [1:p]$, $\vv{\boldsymbol{X}}_S$ denotes the sub-vector of $\vv{\boldsymbol{X}}$ indexed by $S$.

	\section{Mixed-Type Stationary Distribution Modeling}

Let $\{\vv{\boldsymbol{X}}_t = (\boldsymbol{X}_{t1}, \dots, \boldsymbol{X}_{tp})^{\top}\}_{t=1}^n$ be the observed $p$-dimensional stationary process (which includes the case of i.i.d.~observations). Our goal is to learn its stationary distribution. In general, $\vv{\boldsymbol{X}}_t$ may include lagged variables, useful in time series applications (see Section~\ref{Sec5}).
We use \(\vv{\boldsymbol{X}} = (\boldsymbol{X}_{1}, \dots, \boldsymbol{X}_{p})^{\top}\) to denote a generic random vector with joint distribution same as $\vv{\mathbf{X}}_{t}$.

\subsection{Mixed-Type Stationary Distributions}\label{Sec2.1}

\begin{figure}[ht]
    \centering
    \includegraphics[width=0.4\textwidth]{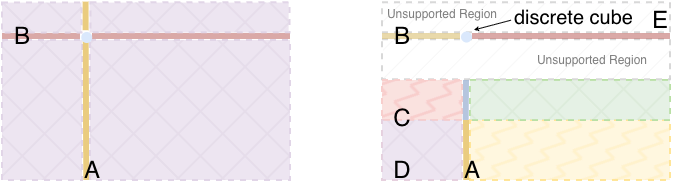}
    \caption{    
    (Left) The two-dimensional feature space is partitioned into four regions due to discrete points $A$ and $B$, including three ``discrete'' regions indexed by $\Theta = \{(\{1\}, A), (\{2\},B), (\{1,2\}, (A, B))\}$ and the rest corresponding to the continuous part. (Right) Hypercube partition of the feature support and bins along each coordinate. The discrete points $A$ and $B$ are themselves bins. There are three bins along the first dimension, and, with $C$ as an additional split, also three bins along the second dimension.
    Together, their Cartesian products constitutes nine hypercubes.
    Regions beyond the boundary points $D$ and $E$ at $(\pm \delta_n, \pm \delta_n)$ are omitted. 
}
   
    \label{fig:bins}
\end{figure}

We consider stationary distributions which are possibly of mixed-type. That is, each variable can be continuous, discrete, or mixed (with continuous and discrete components). In general, the distribution of $\vv{\boldsymbol{X}}$ can be defined by several (sub-probability) density functions (given in \eqref{stationary.dist} below).
Let $D_j$ denote the set of discrete points for the $j$-th coordinate, and define $\Theta = \{ (S, (d_j)_{j \in S}) \big| S \subseteq [1:p],\; d_j \in D_j \ \text{for all } j \in S \}$.
Then, the feature space can be partitioned into subsets indexed by $(S, (d_j)_{j \in S}) \in \Theta$, and a subset in which all coordinates are continuous.
See the left panel of Figure~\ref{fig:bins} for an example. 
To precisely define the densities, we introduce the following decomposition for every measurable $\mathcal{A} \subseteq  \mathbb{R}^p$
and each $(S, \vv{d}_S)\in \Theta$, 
 \begin{equation*}
    \begin{split}
    \mathcal{A}(S, \vv{d}_S) 
& =  \mathcal{A} 
\cap W_1(S, \vv{d}_S)
\cap W_2(S, \vv{d}_S),
    \end{split}
\end{equation*}
with $W_1(S, \vv{d}_S)  = \cap_{j \in S} \{ \vv{v} \in \mathbb{R}^p : v_j = d_j \}$,
$W_2(S, \vv{d}_S)  = \cap_{j \in S^c} \{ \vv{v} \in \mathbb{R}^p : v_j \not\in D_j \}$, $S^c \coloneqq [1:p]\setminus S$, and 
$\mathcal{A}_c = \mathcal{A} \setminus \big(\cup_{(S, \vv{d}_S) \in \Theta} \mathcal{A}(S, \vv{d}_S)\big)$. Then the stationary distribution of $\vv{\boldsymbol{X}}$ is represented by 
$\mathbb{P}(\vv{\boldsymbol{X}} \in \mathcal{A} ) 
 = \mathbb{P}(\vv{\boldsymbol{X}}\in \mathcal{A}_c) + \sum_{(S, \vv{d}_S) \in \Theta} \mathbb{P}(\vv{\boldsymbol{X}}\in \mathcal{A}(S, \vv{d}_S))$, where
\begin{equation}\label{stationary.dist}
\begin{split}
\mathbb{P}(\vv{\boldsymbol{X}}\in \mathcal{A}_c) & = \int_{\mathcal{A}_c} \pi_c(\vv{x}) \, d\vv{x},\\
\frac{\mathbb{P}(\vv{\boldsymbol{X}}\in \mathcal{A}(S, \vv{d}_S) )}{ \mathbb{P}(\vv{\boldsymbol{X}}_S = \vv{d}_S)}&
=  \int_{\vv{u} \in \bar{\mathcal{A}}(S, \vv{d}_S)} \phi_{S, \vv{d}_S}(\vv{u}) \, d\vv{u},
   \end{split}
\end{equation}
in which $\bar{\mathcal{A}}(S, \vv{d}_S)
 = \{ \vv{a}_{S^c} : \vv{a}\in \mathcal{A}(S, \vv{d}_S) \}$,  and let $\pi_c : \mathbb{R}^p \to [0,\infty)$ and $\phi_{S,\vv{d}_S} : \mathbb{R}^{\texttt{\#} S^c} \to [0,\infty)$ denote the corresponding conditional density functions. For notational completeness, the second integral in \eqref{stationary.dist} equals 1 when $S^c = \emptyset$, in which case $\phi_{S, \vv{d}_S} \equiv 1$. We further define the extended conditional density $\pi_{S, \vv{d}_S} : \mathbb{R}^p \to [0, \infty)$ such that $\pi_{S, \vv{d}_S}(\vv{x}) = \phi_{S, \vv{d}_S} (\vv{x}_{S^c})$ for all $\vv{x} \in \mathbb{R}^p$, rendering it invariant to the discrete coordinates indexed by $S$.

\subsection{Bayesian Network Distribution Regression}\label{Sec2.2}


We employ a discretization strategy to learn the mixed-type, multivariate stationary distribution.
First, each coordinate is partitioned into $m_j$ bins and the one-hot vectors \(\vv{\boldsymbol{H}}_j = (\boldsymbol{H}_{j1}, \dots, \boldsymbol{H}_{jm_j})^{\top}\) encodes which bin $\mathbf{X}_{j}$ lies in.
Note that the $m_j$th bin is reserved to represent the tails
\(I_j \setminus [-\delta_n, \delta_n]\), where \(I_j\) denotes the support of \(\boldsymbol{X}_j\) with  $\mathbb{P}(\boldsymbol{X}_j \in I_j) = 1$, and the remaining $(m_j - 1)$ bins form a partition of \( I_j \cap [-\delta_n, \delta_n] \).
In addition, we also treat the discrete points in $D_j$ as independent bins.
Then, the Cartesian products of the bins in each coordinate constitute the resulting $p$-dimensional hypercubes, which form a partition of the feature support.
An example is illustrated in the right panel of Figure~\ref{fig:bins}.
These hypercubes are used in our method (to be introduced below) through the one-hot encodings. For now, we treat them as given. The optimal bin sizes are discussed in Section~\ref{Sec5.c}, and we also tune them using some practical methods in Sections~\ref{Sec4} and \ref{Sec5}.

We now introduce the Bayesian network representation~\citep{jordan1999introduction, bengio1999modeling} which is fundamental to our method. Specifically, for every \( (\vv{h}_1, \dots, \vv{h}_p) \in \{0,1\}^{\sum_{l=1}^p m_l} \), where each \( \vv{h}_l \) is a one-hot vector, it holds that
\begin{align} \label{band}
\begin{split}
    &\mathbb{P}\big( \bigcap_{j=1}^p \{\vv{\boldsymbol{H}}_j = \vv{h}_j\} \big) \\ 
    =&  \prod_{j=1}^p \mathbb{P}\big( \vv{\boldsymbol{H}}_j = \vv{h}_j \ \big| \  \bigcap_{l=1}^{j-1} \{\vv{\boldsymbol{H}}_l = \vv{h}_l\} \big) \\
     =:& \prod_{j=1}^p e_{j,\omega_j(\vv{h}_j)}(\vv{h}_{j-1}, \dots, \vv{h}_1), 
\end{split}
\end{align}
where \( \omega_j : \{0,1\}^{m_j} \to \{1,\dots,m_j\} \) satisfies \( \omega_j(\vv{h}_j) = q \) if the $q$th coordinate of $\vv{h}_j$ is 1. For notational completeness, we write $\mathbb{P}( \vv{\boldsymbol{H}}_1 = \vv{h}_1 \ | \  \vv{\boldsymbol{H}}_0 = \vv{h}_0 ) = \mathbb{P}( \vv{\boldsymbol{H}}_1 = \vv{h}_1)$ and $e_{1l}(\vv{h}_0) = \mathbb{E}(\boldsymbol{H}_{1l})$. 
Now, we can estimate each $e_{jl}$ using high-dimensional regression techniques such as decision trees (\citealp{morgan1963problems}; see Section~\ref{Sec5.2c}), CART~\citep{breiman2017classification}, or Lasso~\citep{tibshirani1996regression}.
Let $\widehat{e}_{jl}$ denote the estimated probability function corresponding to $e_{jl}$.
Then the estimated probability $\mathbb{P}(\vv{\boldsymbol{X}} \in \mathcal{A})$ for a measurable set $\mathcal{A} \subset \mathbb{R}^p$ is then defined as
\begin{equation}
    \label{estimate.stationary.dist}
    \widehat{\mu}(\mathcal{A}) = \sum_{\vv{h}_{1:p}} 
\nu_{\vv{h}_{1:p}}(\mathcal{A} )
\prod_{j=1}^p \widehat{e}_{j,\omega_j(h_j)}(\vv{h}_{j-1 : 1}),
\end{equation}
where the sum is over all one-hot vectors $\vv{h}_j \in \{0,1\}^{m_j}$, and that $\nu_{\vv{h}_{1:p}}$ is a normalized Lebesgue measure with respect to the continuous coordinates of $\text{Cube}(\vv{h}_{1:p})$, the hypercube corresponding to the $\omega_j(h_j)$-th bin on the $j$-th coordinate, which has total mass one and zero outside the hypercube.
For example, for the hypercube whose second coordinate is fixed at $B$ in Figure~\ref{fig:bins}, its continuous coordinate is the first coordinate. In addition, if $\text{Cube}(\vv{h}_{1:p})$ is a discrete cube, then $\nu_{\vv{h}_{1:p}}(\mathcal{A}) = 1$ if $\text{Cube}(\vv{h}_{1:p}) \subseteq \mathcal{A}$ and $\nu_{\vv{h}_{1:p}}(\mathcal{A}) = 0$ otherwise. 
Since this method relies on the Bayesian network representation as well as modeling multiple probability functions via regressions, we call it the BAyesian Network Distribution (BAND) regression.

Many existing methods for discrete distribution modeling use Bayesian networks via maximizing the log-likelihood function, using shallow neural networks such as sigmoid belief networks~\citep{neal1992connectionist, jordan1999introduction}, or deep architectures such as PixelRNN~\citep{van2016pixel} and sparse transformers~\citep{child2019generating}. 
In contrast, the proposed BAND regression models the distribution by directly estimating each conditional probability function $e_{jl}(\vv{\boldsymbol{H}}_{j-1}, \dots, \vv{\boldsymbol{H}}_1) = \mathbb{E}\big(\boldsymbol{H}_{jl} \mid \vv{\boldsymbol{H}}_{j-1}, \dots, \vv{\boldsymbol{H}}_1\big)$. A key advantage of this formulation is that BAND explicitly handles variable sparsity structure using the well-known high-dimensional regression techniques, which provides provable guarantees that avoid the curse of dimensionality.

\section{Theoretical Foundations of Distribution Modeling}\label{Sec5.c}

\subsection{Convergence Rates in Total Variation Distance}\label{Sec4.2}

Our main goal in this subsection is to derive a convergence rate in total variation distance (Theorem~\ref{theorem2}) for BAND.
To this end, let us first introduce some high-level assumptions.

\begin{condi}\label{mean.regression}

There is $d_{\min} > 0$ such that $e_{jl}(\vv{h}_{j-1 : 1}) \ge d_{\min}$ 
if $\omega_1(\vv{h}_{1}) < m_1, \dots, \omega_{j-1}(\vv{h}_{j-1}) < m_{j-1}$ and $l< m_j$. Additionally, $\mathbb{P} (E^{\dagger}) \le \Delta$ for some $\Delta >0$, where
$E^{\dagger}  = \cap_{1\le j\le p} \cap_{ 1\le l< m_j} \cap_{\omega_q(\vv{h}_{q}) < m_q, 1\le q< j } \big\{ \big| e_{jl}(\vv{h}_{j-1 : 1}) -  \widehat{e}_{jl}(\vv{h}_{j-1: 1}) \big| \le \sqrt{T}\big\}$. 
\end{condi}

Condition~\ref{mean.regression} assumes a deviation upper bound $\sqrt{T}$ for the regression estimates, and a probability lower bound $d_{\min}$ for the regression functions.
These parameters can be estimated under more fundamental assumptions (see Corollaries~\ref{corollary1} and \ref{corollary2}), with $\sqrt{T}$ depending only mildly on $p$ thanks to the sparse Bayesian networks (Condition~\ref{sparsity}).

In Condition~\ref{conti} below, $\Theta^{\star} = \Theta$ if $\vv{\boldsymbol{X}}$ is discrete; otherwise, $\Theta^{\star} = \Theta \cup \{c\}$ includes an additional label $c$. Recall that $\pi_c$ is defined in~\eqref{stationary.dist} while $\pi_\theta$, $\theta \in \Theta$, are defined after~\eqref{stationary.dist}. 

\begin{condi}\label{conti}
There exists some $R>0$ such that for each $\theta \in \Theta^{\star}$ and every $\vv{x}$ with $\|\vv{x}\|_{\infty} \le R$, it holds that
$\pi_{\theta}(\vv{x})>0$, $\log \pi_{\theta}(\vv{x}) \in C^{1}(\vv{x})$, and
$\|\nabla \log \pi_{\theta}(\vv{x})\|_2 \le K_{1}\sqrt{p}\,R$ for some constant $K_{1} > 0$.
\end{condi}

A bounded gradient of the log-density (the \textit{score}) is a standard regularity assumption in the density estimation literature, including in score-based diffusion  models. 
Condition~\ref{conti} additionally accounts for the dependence on the feature dimension $p$. It can be verified that the (boundary adjusted) uniform, Student-$t$, and multivariate Gaussian distributions and their mixtures all satisfy Condition~\ref{conti}; see Section~\ref{example.distribution} for details. Alternatively, one may assume
$\|\nabla \log \pi_{\theta}(\vv{x})\|_{\infty} \le K_{1} R$, which implies Condition~\ref{conti} since $\|\nabla \log \pi_{\theta}(\vv{x})\|_{2}\le \sqrt{p}\|\nabla \log \pi_{\theta}(\vv{x})\|_{\infty}$.

Let $\{\delta_n\}$ be a positive sequence, and let $\varepsilon_n \ge 0$ denote the maximum diameter of the cubes in the partition of $[-\delta_n, \delta_n]^p$ (cubes as defined after \eqref{estimate.stationary.dist}). 
Now we are ready to state the main result, Theorem \ref{theorem2}, whose proof is provided in Section~\ref{proof.theorem1}.
Note that $(\Delta, d_{\min}, T, E^{\dagger}, \varepsilon, p)$ all depend on $n$; however, this dependence is omitted for brevity.

\begin{theorem}\label{theorem2}
 Set \(\widehat{e}_{jl}(\vv{h}_{j-1}, \dots, \vv{h}_1) = 0\) 
if \(\omega_q(\vv{h}_q) = m_q\) for some \(1 \le q < j\) or \(l = m_j\). 
For all large $n$, assume Condition~\ref{mean.regression}--\ref{conti} with $R = \delta_n$. 
Assume also $\varepsilon \le (\delta_n\sqrt{p}K_{1})^{-1}$ and $e p \sqrt{T} \leq d_{\min}$ for all large $n$.
Then, it holds for all large $n$ that, with probability at least \(1 - \Delta\),
{\small
\begin{equation*}
\sup_{\mathcal{A}\in \mathcal{R}^p}\bigl| \mathbb{P}(\vv{\boldsymbol{X}} \in \mathcal{A}) - \widehat{\mu}(\mathcal{A}) \bigr| \le  \frac{e p \sqrt{T}}{d_{\min}} + \varepsilon   \sqrt{p} \delta_n K_{1} e + Q_{\delta_n},
\end{equation*}
}where $Q_{\delta_n} = \mathbb{P}(\vv{\boldsymbol{X}}\not\in [-\delta_n, \delta_n]^p )$.
\end{theorem}


\begin{remark}
    The first two terms in the upper bound in Theorem~\ref{theorem2} correspond to estimation variance and bias, respectively.
    In Section~\ref{proof.theorem1}, we also provide a similar result for conditional distribution.     
\end{remark}

One major implication of Theorem \ref{theorem2} is that BAND can consistently estimate high-dimensional distributions, which allows $p$ to grow polynomially with the sample size $n$.
To illustrate, in the following subsections we consider the special cases of continuous and discrete distributions in Corollaries~\ref{corollary1} and~\ref{corollary2}, respectively.
Across both cases, we achieve polynomial convergence rates in $n$, allowing the feature dimension $p$ to grow polynomially with the sample size. In contrast, existing theoretical results achieving optimal rates~\citep{cevid2022distributional, liu2023convergence, oko2023diffusion, zhang2024minimax, biroli2024kernel} typically require $p \le \log n$ for total variation consistency.

\subsection{Application to Continuous Multivariate Time series}
\label{Sec5.2c}



Suppose $\{\vv{\boldsymbol{X}}_t\}_{t=1}^{n}$ is a continuous stationary process with $\Theta = \emptyset$ (that is, $\int_{\mathbb{R}^p} \pi_c(\vv{x})\, d\vv{x} = 1$) and feature supports $I_1 = \cdots = I_p = \mathbb{R}$. Let 
$\vv{\boldsymbol{H}}_{tj} = (\boldsymbol{H}_{tj1}, \dots, \boldsymbol{H}_{tjm_j})^{\top}$, $t \in [1:n]$, denote the sample counterparts of $\vv{\boldsymbol{H}}_j$, i.e., the corresponding discretized sample constructed as in Section~\ref{Sec2.2}. 
To illustrate how BAND adapts to the sparse Bayesian network, consider, for each $1 \le j \le p$ and $1 \le l < m_j$, a sparse tree model
{\small\begin{equation}
    \label{tree}
    \widehat{e}_{jl}(\vv{h}_{j-1}, \dots, \vv{h}_1) 
    = \frac{\sum_{t=1}^n \boldsymbol{H}_{tjl} \prod_{q \in S_j} \boldsymbol{1}\{\vv{\boldsymbol{H}}_{tq} = \vv{h}_q\}}
           {\sum_{t=1}^n \prod_{q \in S_j} \boldsymbol{1}\{\vv{\boldsymbol{H}}_{tq} = \vv{h}_q\}}
\end{equation}}%
for $(\vv{h}_{1}, \dots, \vv{h}_p) \in \{0,1\}^{m_1} \times \dots \times \{0,1\}^{m_p}$ with one-hot $\vv{h}_j$ and $\omega_j(\vv{h}_j) < m_j$; otherwise, the tree model prediction is defined to be zero. Here, $S_1, \dots, S_p$ are relevant sets given by Condition~\ref{sparsity} below. 


\begin{condi}[Sparse Bayesian Networks]\label{sparsity}
    \textnormal{(a)} For each $j$, there is some $S_j\subset \{1, \dots, j-1\}$ such that $\vv{\boldsymbol{H}}_j$ is independent of $(\vv{\boldsymbol{H}}_l, l\not\in S_j, l<j)$ conditional on $(\vv{\boldsymbol{H}}_l, l\in S_j)$. \textnormal{(b)} There are $S_{1}, \dots, S_p$  such that $ S_{j} \subset \{1, \dots, j - 1\}$ and that $\vv{\boldsymbol{X}}_{S_{j}\cup\{j\}}$ is independent of  $\vv{\boldsymbol{X}}_{\{1, \dots, j-1\}\backslash S_{j}}$.
\end{condi}

Condition~\ref{sparsity} is trivially satisfied when the Bayesian network is not sparse (\(S_j = \{j-1,\dots,1\}\)). Condition~\ref{sparsity}(a) requires the discretized one-hot vectors to satisfy an autoregressive conditional independence property, which does not automatically follow from $\boldsymbol{X}_1, \dots, \boldsymbol{X}_p$ having the same property. However, it is a natural condition in many applied fields, such as economics, finance, and engineering, where continuous variables are routinely discretized. For example, \citet{tauchen1986finite} shows that increasingly fine partitions yield accurate approximations of continuous processes by finite-state Markov chains. 
Additionally, Example~\ref{k_modals_unif} in Section \ref{Sec4} illustrates that, with sufficiently many discretization splits, Condition~\ref{sparsity}(a) holds for some mixtures of uniform distributions. 
More examples can be found in~\citep{vandermeulen2024breaking}.
Alternatively, Condition~\ref{sparsity}(b) assumes there are multiple groups of mutually independent variables, which is another form of sparse Bayesian networks.

With Condition~\ref{sparsity}, BAND combined with the sparse tree models \eqref{tree} can potentially attain a faster rate of convergence; see Corollaries~\ref{corollary1} and~\ref{corollary2}, which are derived assuming these relevant sets are known. A formal analysis of recovering these sets via sparse learning is left for future work.
In practice, the relevant sets $S_j$ are unknown and can be estimated using sparse regression methods such as Lasso~\citep{tibshirani1996regression},
CART~\citep{breiman2017classification}, and CatBoost~\citep{prokhorenkova2018catboost}.


Finally, we need a technical, though standard, assumption on the weak dependence in the data in order to establish concentration inequalities.

\begin{condi}[Strong Mixing]
    \label{alpha}    
    There are constants \( \gamma_0, \gamma_1 > 0 \) such that the sequence \( \{ (\boldsymbol{X}_{tl}, l\in S_j\cup\{j\})\}_{t} \) is strongly mixing with mixing coefficients satisfying $\alpha(k) \le \gamma_0 e^{-\gamma_1 k}$ for all $k \ge 1$ and all $1\le j\le p$, with $S_j$ specified in Condition~\ref{sparsity}.
\end{condi}
A formal definition of the strong mixing property is provided in Section~\ref{SecB.1}. Condition~\ref{alpha} requires that each sub-process $\{ (\boldsymbol{X}_{tl}, l \in S_j \cup \{j\}) \}_t$ is strong mixing, allowing linear AR, threshold AR, bounded nonlinear AR, ARCH, GARCH, and VAR processes~\citep{an1996geometrical, tsay2005analysis}. 
This is slightly weaker than requiring the full process to be strongly mixing. Now we can state Corollary~\ref{corollary1}, whose proof is given in Section~\ref{SecA.2}.

\begin{corollary}
\label{corollary1}    
Let \(\delta_n > 0\) be arbitrary.
Assume Condition~\ref{conti} (with \(R = \delta_n\)) and Condition~\ref{alpha}, and either Condition~\ref{sparsity}(a) or \ref{sparsity}(b) with \(\max_{1 \le j \le p} \#S_j \le s_0\) for some constant \(s_0 \ge 0\).
In addition, suppose the joint density of \(\vv{\boldsymbol{X}}_J\) is uniformly bounded above and below whenever \(\#J \le s_0 + 1\). 
If the splitting scheme of the bins satisfies \(\max_{1 \le j \le p} m_j \le n\), \(\sup_{n \ge 1}(L_{2n}/L_{1n}) < \infty\), \(L_{1n}^{-1-s_0} \le n\), and \(\lim_{n \to \infty} L_{2n} = 0\), where \(L_{1n}\) and \(L_{2n}\) denote the minimum and maximum bin lengths, and if $pn^{-1/2} (\log{n})^5 L_{2n}^{-\frac{1}{2} - \frac{s_0}{2} - b_0}\le 1$ and $p L_{2n} \delta_n =o(1)$ for some $b_0 > 0$, then with probability tending to one, 
$\sup_{\mathcal{A} \in \mathcal{R}^p}\big| \mathbb{P}(\vv{\boldsymbol{X}} \in \mathcal{A}) - \widehat{\mu}(\mathcal{A}) \big|  \le  
        pn^{-\frac{1}{2}} (\log{n})^5 L_{2n}^{-\frac{1}{2} - \frac{s_0}{2} - b_0}  
        +  p L_{2n} \delta_n  \log{n}+ Q_{\delta_n}$.
\end{corollary}

Note that Corollary~\ref{corollary1} allows the feature dimension $p$ to grow polynomially with the sample size, while maintaining polynomial convergence rates.  Leveraging the sparsity in Condition~\ref{sparsity}, we avoid the curse of dimensionality and achieve a convergence rate far faster than the existing results for adaptive partitioning~\citep{liu2023convergence}, distributional random forests~\citep{cevid2022distributional}, and score-based diffusion models~\citep{oko2023diffusion, zhang2024minimax}, when \(p \gg \log n\). In Example \ref{rates} below, we explicitly calculate the polynomial convergence rate with a carefully selected bin size. 

\begin{exmp}\label{rates}
Let $L_{2n} = n^{-1/(3 + s_0 + 2b_0)}$ and $\delta_n = n^{\epsilon_1}$ for some $\epsilon_1 > 0$, assuming $\pi_c$ satisfies sufficiently strong polynomial moment conditions so that $Q_{\delta_n}$ is negligible. Then the convergence rate in Corollary~\ref{corollary1} is $p \, n^{-1/(3 + s_0 + 2b_0) + \epsilon_1} (\log n)^5$.
\end{exmp}

    The rate $n^{-1/(2+p)}$ is the established optimal for multivariate histogram estimators \citep{beirlant1998lrerror}. By comparison, the rate derived in Example~\ref{rates} circumvents the curse of dimensionality by replacing the dimension $p$ in the exponent with the sparsity-related term $1 + s_0 \le p$ (where $s_0 = 0$ corresponds to the case $S_j = \emptyset$ for all $j$), with an additional linear factor of $p$. \citet{vandermeulen2024breaking} obtained convergence rates comparable to that in Example~\ref{rates}, specifically by replacing the dimension $p$ in the exponent with a sparsity parameter under a condition similar to Condition~\ref{sparsity}. However, their results depend on densities supported exclusively on the unit hypercube $[0, 1]^p$. Such confined support is generally incompatible with time series data, which typically requires a broader domain, thus limiting their findings to independent data settings.

Besides sparse Bayesian networks, prior literature on copula models~\citep{nagler2016evading, oh2017modeling} have considered other dependence structures for multivariate distribution modeling, such as the simplified vine copulas.
However, those assumptions are rather restrictive, which, for example, exclude some multi-modal mixture distributions.
In contrast, the sparse Bayesian network condition employed here has greater flexibility.

\subsection{Application to High-Dimensional Discrete Data}\label{Sec3.3}

In this subsection, we consider a discrete process $\vv{\boldsymbol{X}}_t$ whose stationary distribution satisfies $\mathbb{P}(\boldsymbol{X}_j \in D_j) = 1$ with $D_j = \{1, \dots, m_j - 1\}$, and define $I_j := D_j \cup \{m_0\}$, where $m_0 = \max_{1 \le j \le p} m_j$ and $\delta_n = m_0 - 1$ are assumed to be constant. Then, by the construction of the bins introduced in Section \ref{Sec2.2}, $\boldsymbol{X}_j = l$ if and only if $\boldsymbol{H}_{jl} = 1$ for $l \in [1:(m_j - 1)]$. Here, all cubes here are discrete cubes, with diameter $\varepsilon =0 $.
Corollary~\ref{corollary2} builds on the same tree model~\eqref{tree}; see Section~\ref{proof.corollary2} for the proof.

\begin{corollary}
\label{corollary2}    
 Let constants \( s_0 \ge 0 \) and \( \gamma_3, \gamma_4, \gamma_5 > 0 \) be given.
Assume Condition~\ref{sparsity}(a) with \(\max_{1 \le j \le p} \texttt{\#}S_j \le s_0\), Condition~\ref{alpha}, and that \(\max_{1 \le j \le p} m_j \le m_0\) and $p n^{-1/2} (\log n)^5\le 1$.
In addition, assume \(\gamma_4 \gamma_3^{\texttt{\#}S} \le \mathbb{E}(\prod_{j \in S}\boldsymbol{H}_{jl_j}) \le \gamma_5 \gamma_3^{\texttt{\#}S}\) for all \(1 \le l_j < m_j\) and \(S \subset \{1, \dots, p\}\) with \(\texttt{\#}S \le s_0 + 1\).
Then, with probability tending to one, 
$\sup_{\mathcal{A} \in \mathcal{R}^p}\big| \mathbb{P}(\vv{\boldsymbol{X}} \in \mathcal{A}) - \widehat{\mu}(\mathcal{A}) \big| \le p n^{-1/2} (\log n)^5$.
\end{corollary}

In general, the minimax optimal rate for estimating discrete distributions without sparsity is $O(\sqrt{\text{Support Size} / n})$ \citep{Han2015minimax}. 
Take a \(p\)-dimensional vector of independent Bernoulli variables as an example. The support size is \(2^p\), which scales poorly with $p$ (in fact, the minimax rate is much slower when the support size increases with $n$). In comparison, Corollary~\ref{corollary2} offers a substantially faster convergence rate that scales linearly with \(p\) by incorporating the sparse Bayesian network structure.

\section{Simulation Experiments}\label{Sec4}

\subsection{Simulation Setup}\label{Sec4.1}

In this section, we first examine the finite-sample performance of BAND when applied to i.i.d.~synthetic data generated from bimodal distributions described in Examples~\ref{k_modals}--\ref{k_modals_unif} below. 
Multimodal distributions arise in many applications, from discrete mixtures in genetics~\citep{friedman2000using} and text classification~\citep{juan2002use} to continuous regime-switching models~\citep{tauchen1986finite, hamilton1989new} and clustering~\citep{melnykov2010finite}, but learning multimodal laws is difficult in high dimensions.
In the following, let $\vv{\mu}_1 = (0,\dots,0)^{\top}$ and $\vv{\mu}_2 = (r,\dots,r)^{\top}$, where $r \ge 0$.

\begin{exmp}[Gaussian mixture] \label{k_modals}
$\vv{\boldsymbol{X}} \sim \sum_{i=1}^{2} \frac{1}{2} \mathcal N_p(\vv{\mu}_i, I_p)$.
\end{exmp}
\begin{exmp}[Uniform mixture] \label{k_modals_unif} 
Let $\vv{\boldsymbol{X}} \sim \sum_{i=1}^{2} \frac{1}{2} \mathcal{U}_p([-\frac{3}{2},$ $\frac{3}{2}]^p$ $+ \vv{\mu}_i)$, where $\mathcal{A} + \vv{v} = \{\vv{x} + \vv{v} \mid \vv{x}\in\mathcal{A} \}$ for $\mathcal{A} \subset \mathbb{R}^p$ and $\vv{v}\in \mathbb{R}^p$.
\end{exmp}

Note that Examples~\ref{k_modals}--\ref{k_modals_unif} satisfy the sparsity Condition~\ref{sparsity}(a) with $S_j = \emptyset$ for $r=0$. 
Meanwhile, for Example~\ref{k_modals_unif} with $r > 3$, Condition~\ref{sparsity}(a) is met with $S_j = \{j-1\}$, provided there are sufficiently many evenly-spaced splits and $\delta_n \ge 3/2 + r$, resulting in two well-separated modes (see Figure~\ref{fig:all_models}); see Section~\ref{proof.sec4.1} for details.  
In general, however, the sparsity condition does not hold, creating challenges for BAND in learning the distribution. 
    
Next, we consider the following block-wise graphical Gaussian distribution. 
\begin{exmp}\label{graphical_gaussian}
Partition $\{1,\dots,p\}$ into $m$ equal-sized blocks $B_1,\dots,B_m$. 
The data follow $\vv{\boldsymbol{X}}_{B_k} \sim \mathcal{N}_{\frac{p}{m}} (\vv{0}, (1-\rho) I_{\frac{p}{m}} + \rho\, \vv{1}\vv{1}^{\top} )$, drawn independently for $k = 1, 2, \ldots, m$, 
where $\vv{1}$ denotes a vector of ones.
\end{exmp}

Note that Condition~\ref{sparsity}(b) holds in Example~\ref{graphical_gaussian}, where a larger $m$ indicates a higher degree of sparsity.
The sparsity pattern is not known to the practitioner.
Nevertheless, as we will see later, BAND can successfully adapt to the sparse structure and estimate the underlying distribution well.

\subsubsection{Implementation and Benchmark Methods}\label{method}
Here we briefly discuss the implementation details and introduce some benchmark methods.

\noindent {\bf BAND}: For continuous variables, as in our applications, we recommend using 
$\lceil \texttt{factor} \times n^{1/3} \rceil$ equally-spaced breaks over the sample range of each coordinate, following the classical cube-root rule for binning~\citep{scott1979optimal}. 
We use $\texttt{factor} =2$ in this section.
The mean regression models are fitted using CatBoost~\citep{prokhorenkova2018catboost} with default parameters due to its minimal hyperparameter tuning \citep{mcelfresh2023neural}. 
When sampling, BAND draws the $j$th variable from a hypercube determined by an $m_{j}$-dimensional probability vector $\vv{p}$ conditional on the preceding $j-1$ variables (see Section~\ref{sampleing}). 
To reduce instability, we introduce a \textit{sampling probability threshold factor} (\texttt{SPTF} $\in [0, 0.4]$), zeroing entries of $\vv{p}$ below $\texttt{SPTF} \times \text{std}(\vv{p})$; similar techniques have been used in generative models for sampling stabilization~\citep{lu2025dpm}.

Since we use CatBoost as the sub-routine for BAND, the computational cost of training BAND is proportional to training
\((p\, n^{1/3})^2\) CatBoost models on an \(n \times p\) feature matrix; see
\citet{prokhorenkova2018catboost} for CatBoost’s complexity. 
In typical economic time-series applications, the number of variables is modest (see Section~\ref{Sec5}),
so the computational cost is manageable.
Moreover, we found that the default learning rates in the range $[0.03, 0.05]$ for CatBoost perform as stably as automatic tuning.
Therefore, we simply fix the learning rate at 0.05 in the experiments.
Finally, note that tuning \texttt{SPTF} does not require retraining the model, which is computationally inexpensive.

\noindent {\bf Gaussian mixture models}~\cite{scrucca2016mclust}: 
Gaussian mixtures models are fitted using the \texttt{mclust} package in \textsf{R}, which implements the expectation--maximization algorithm with the number of components fixed at $K=2$.

\noindent {\bf Vine Copula}~\citep{nagler2016evading}: 
We implement the vine copula model via the \texttt{kdvine} package in \textsf{R}, which was proposed to estimate multivariate densities.

\noindent {\bf Normalizing Flow}~\citep{durkan2019neural}:
Normalizing flows are mature and well-validated generative models, which implicitly learns the multivariate density. Compared to the data-hungry score-based diffusion models~\citep{song2021scorebased, zhang2025training}, it offers a practical and strong benchmark.
We implement normalizing flows using the \texttt{nflows} library~\citep{nflows}.
To maintain robustness in our moderate-sample setting, we vary only the number of flow layers and use a standard 70/30 training-validation split for model selection, keeping all other hyperparameters at the library’s stable defaults.

\subsection{Results}\label{Sec4.2b}

To assess the performance of the employed methods, we use the Cram\'{e}r--Wold (CW) distance~\citep{cramer1936some}, a computable lower-bound surrogate for the total variation distance; details are provided in Section~\ref{l1measure}.
Tables~\ref{tab:gaussian_mixture}--\ref{tab:sparsity} 
show the resulting CW distances for various $(p, n, r)$ configurations.

For Example~\ref{k_modals}, \texttt{mclust} is correctly specified and therefore serves as the ideal benchmark method. As shown in Table~\ref{tab:gaussian_mixture}, \texttt{mclust} yields CW errors that are among the lowest. Nevertheless, BAND fares quite well, which significantly outperforms \texttt{nflows} in the high dimension case $p = 30$.

For Example~\ref{k_modals_unif} in Table~\ref{tab:uniform_mixture}, BAND yields the lowest CW errors when $n = 1000$.
In particular, when the two supports in the uniform mixture are well separated ($r = 4$), the CW error of BAND (around 0.055) is 75\% smaller compared to \texttt{kdvine} (around 0.20--0.25). 
Indeed, Figure~\ref{fig:all_models} shows that BAND accurately recovers the uniform mixture distribution, while \texttt{kdvine} is unsuitable for this case. \texttt{nflows} also incurs high CW errors despite its flexibility in distribution estimation. 
We note that although \texttt{mclust} is misspecified in this example, its CW errors in Table~\ref{tab:uniform_mixture} do not appear particularly poor.

{\renewcommand{\arraystretch}{0.9}

\begin{table}[h!]
\centering
\scriptsize
\caption{Average Cram\'{e}r--Wold distance over 30 trials for the four methods. Data are generated from the Gaussian mixture in Example~\ref{k_modals}. Standard deviations are in parentheses.}
\label{tab:gaussian_mixture}
\begin{tabular}{l l rrr}
\toprule
Method & $p$/$n$ & $r=0$ & $r=2$ & $r=4$ \\
\midrule
\multirow{4}{*}{BAND} 
 & 10/100 & 0.169 (0.02) & 0.158 (0.02) & 0.161 (0.03) \\
 & 10/1000 & 0.055 (0.01) & 0.070 (0.01) & 0.059 (0.01) \\
 & 30/100 & 0.191 (0.02) & 0.180 (0.02) & 0.181 (0.02) \\
 & 30/1000 & 0.062 (0.01) & 0.077 (0.01) & 0.062 (0.01) \\
\midrule
\multirow{4}{*}{mclust} 
 & 10/100 & 0.131 (0.03) & 0.117 (0.02) & 0.134 (0.02) \\
 & 10/1000 & 0.057 (0.01) & 0.055 (0.01) & 0.057 (0.01) \\
 & 30/100 & 0.153 (0.02) & 0.151 (0.03) & 0.147 (0.02) \\
 & 30/1000 & 0.067 (0.01) & 0.065 (0.01) & 0.063 (0.01) \\
\midrule
\multirow{4}{*}{nflow} 
 & 10/100 & 0.202 (0.04) & 0.182 (0.02) & 0.187 (0.02) \\
 & 10/1000 & 0.087 (0.01) & 0.090 (0.01) & 0.106 (0.02) \\
 & 30/100 & 0.409 (0.05) & 0.358 (0.05) & 0.280 (0.04) \\
 & 30/1000 & 0.087 (0.01) & 0.093 (0.01) & 0.122 (0.02) \\
\midrule
\multirow{4}{*}{kdvine} 
 & 10/100 & 0.166 (0.02) & 0.174 (0.03) & 0.219 (0.03) \\
 & 10/1000 & 0.076 (0.01) & 0.150 (0.01) & 0.198 (0.01) \\
 & 30/100 & 0.176 (0.02) & 0.202 (0.02) & 0.260 (0.02) \\
 & 30/1000 & 0.090 (0.01) & 0.183 (0.01) & 0.244 (0.01) \\
\bottomrule
\end{tabular}
\end{table}

\begin{table}[h!]
\centering
\scriptsize
\caption{Average Cram\'{e}r--Wold distance over 30 trials for the four methods. Data are generated from the uniform mixture in Example~\ref{k_modals_unif}. Standard deviations are in parentheses.}
\label{tab:uniform_mixture}
\begin{tabular}{l l rrr}
\toprule
Method & $p$/$n$ & $r=0$ & $r=2$ & $r=4$ \\
\midrule
\multirow{4}{*}{BAND} 
 & 10/100 & 0.182 (0.02) & 0.157 (0.02) & 0.147 (0.02) \\
 & 10/1000 & 0.059 (0.01) & 0.063 (0.01) & 0.054 (0.01) \\
 & 30/100 & 0.195 (0.02) & 0.184 (0.02) & 0.164 (0.02) \\
 & 30/1000 & 0.065 (0.01) & 0.069 (0.01) & 0.059 (0.01) \\
\midrule
\multirow{4}{*}{mclust} 
 & 10/100 & 0.137 (0.02) & 0.114 (0.02) & 0.114 (0.02) \\
 & 10/1000 & 0.088 (0.01) & 0.072 (0.01) & 0.061 (0.00) \\
 & 30/100 & 0.148 (0.01) & 0.147 (0.02) & 0.139 (0.02) \\
 & 30/1000 & 0.092 (0.00) & 0.076 (0.01) & 0.065 (0.01) \\
\midrule
\multirow{4}{*}{nflow} 
 & 10/100 & 0.188 (0.02) & 0.181 (0.03) & 0.185 (0.03) \\
 & 10/1000 & 0.103 (0.02) & 0.088 (0.01) & 0.113 (0.02) \\
 & 30/100 & 0.375 (0.05) & 0.318 (0.04) & 0.254 (0.05) \\
 & 30/1000 & 0.116 (0.01) & 0.102 (0.01) & 0.124 (0.02) \\
\midrule
\multirow{4}{*}{kdvine} 
 & 10/100 & 0.167 (0.03) & 0.167 (0.01) & 0.212 (0.01) \\
 & 10/1000 & 0.079 (0.01) & 0.161 (0.01) & 0.206 (0.01) \\
 & 30/100 & 0.175 (0.02) & 0.195 (0.01) & 0.257 (0.01) \\
 & 30/1000 & 0.089 (0.01) & 0.190 (0.01) & 0.251 (0.01) \\
\bottomrule
\end{tabular}
\end{table}
}

\begin{figure}[ht]
\centering


\subfloat[BAND]{
    \includegraphics[width=0.23\textwidth]{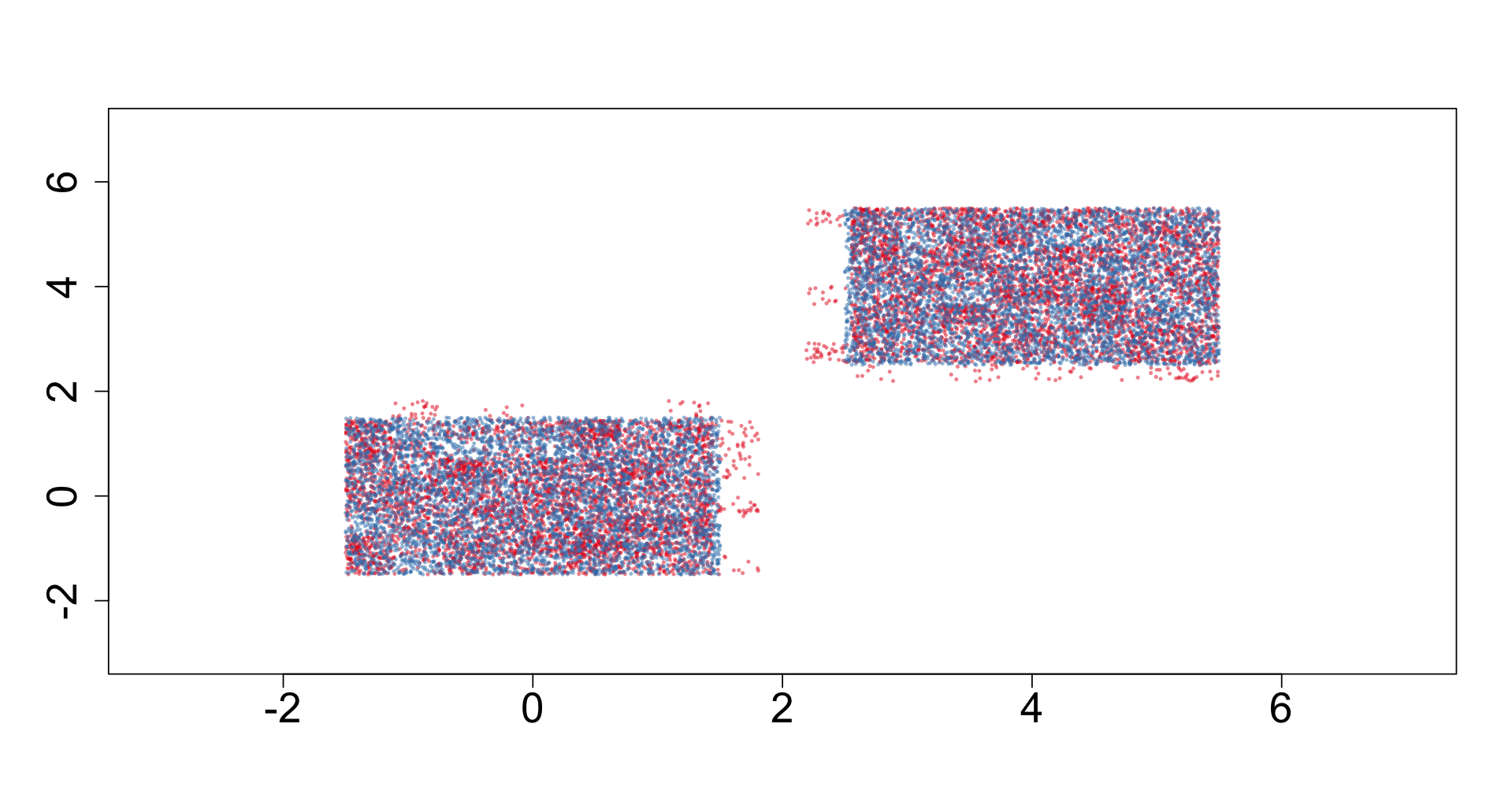}
}
\hfill
\subfloat[mcluster]{
    \includegraphics[width=0.23\textwidth]{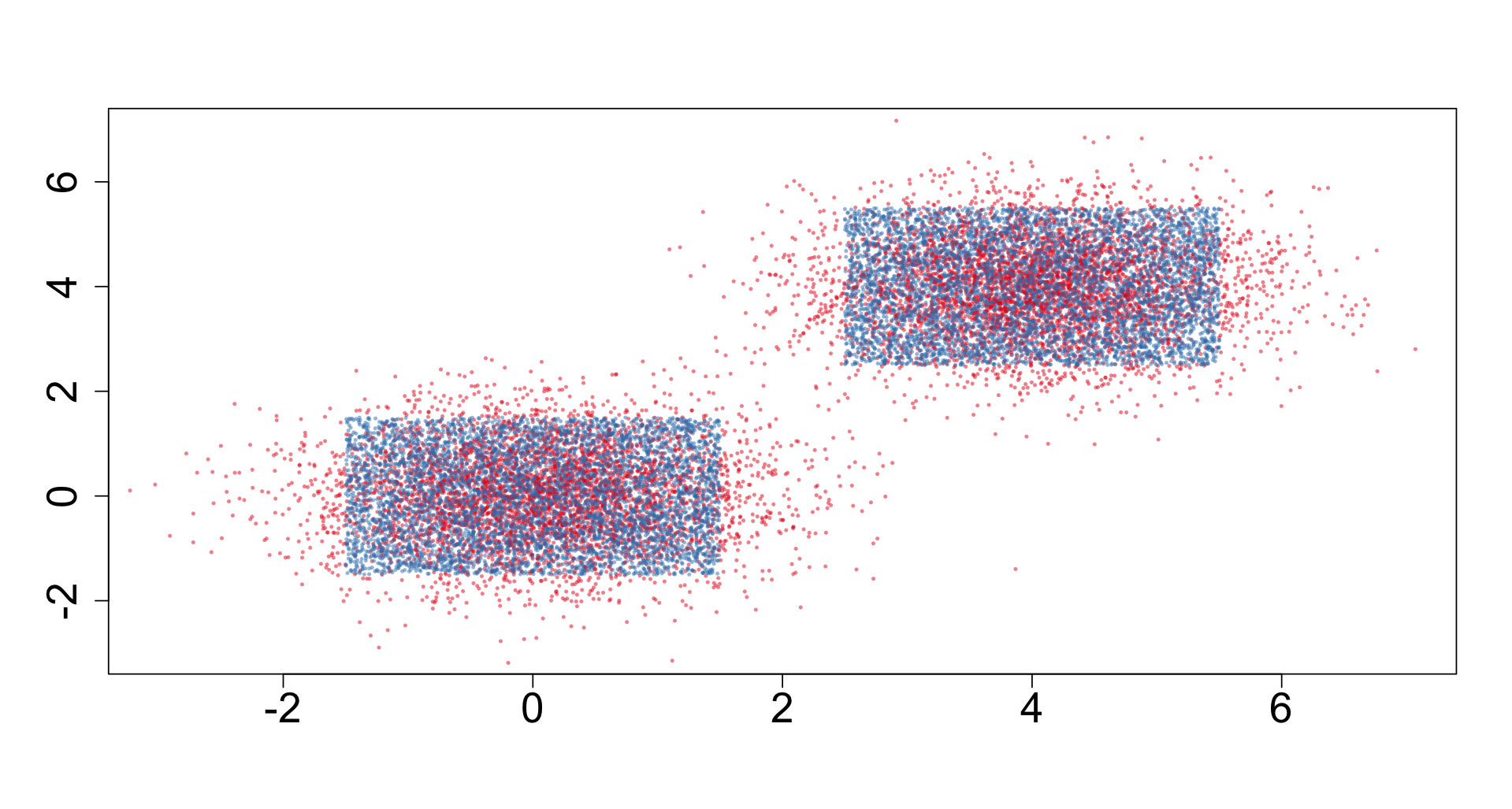}
}
\hfill
\subfloat[nflow]{
    \includegraphics[width=0.23\textwidth]{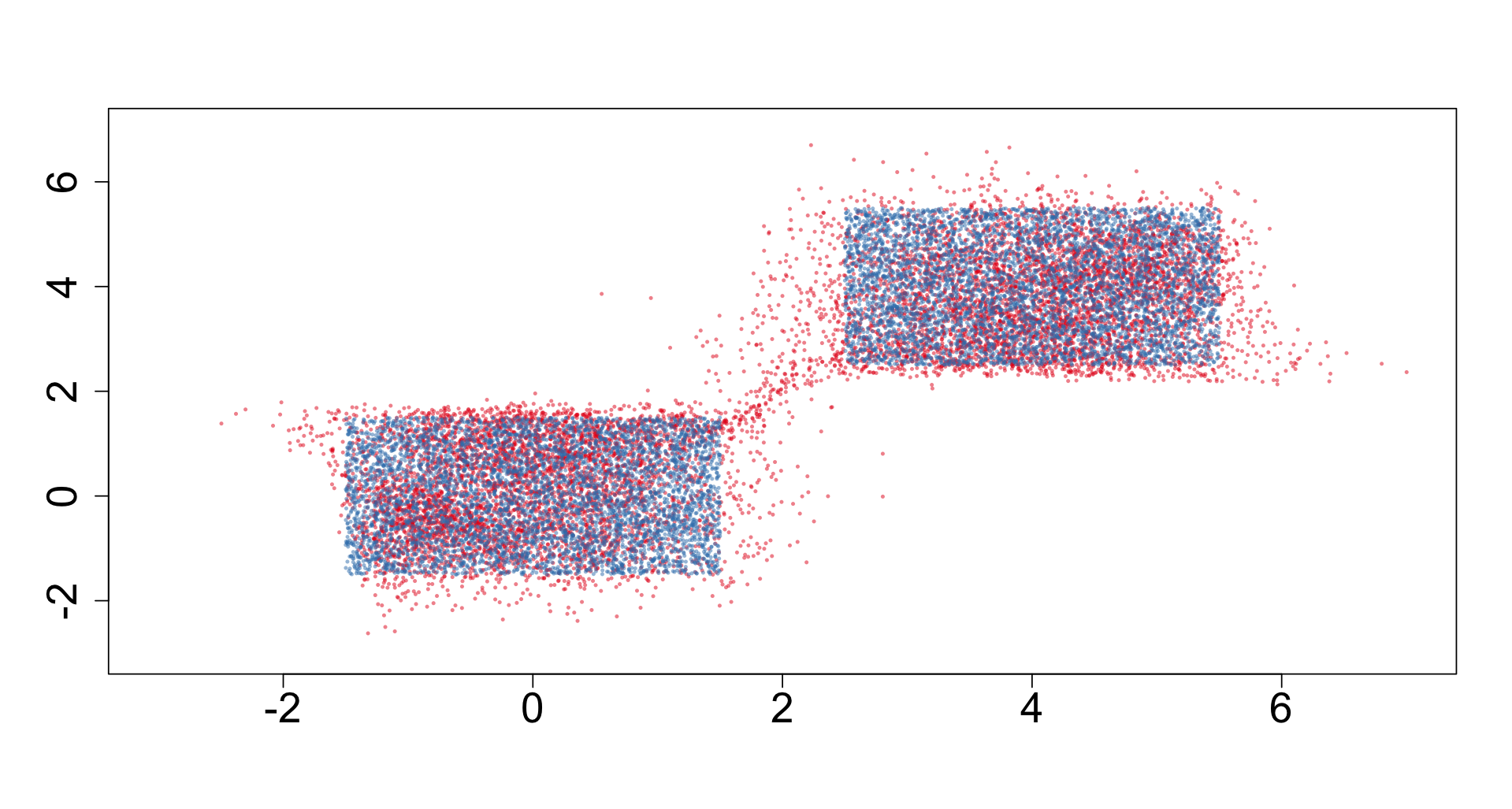}
}
\hfill
\subfloat[kdvine]{
    \includegraphics[width=0.23\textwidth]{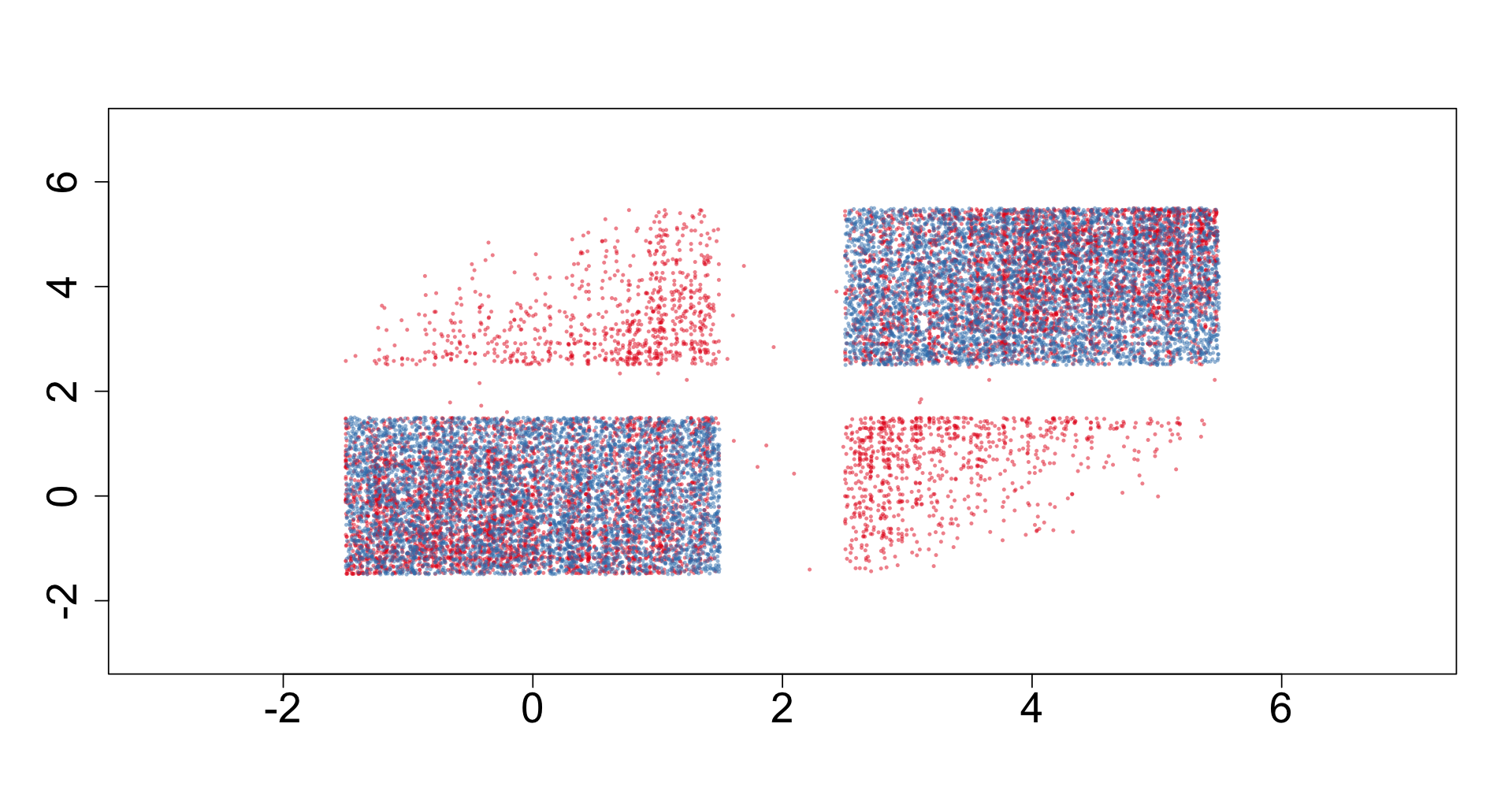}
}
\caption{Scatterplots of 10,000 empirical data points (blue) and 10,000 points sampled from the four fitted models (red). Data are drawn from Example~\ref{k_modals_unif} with $r = 4$. 
Each method is trained on $n=1000$ samples with $p=2$.}
\label{fig:all_models}
\end{figure}

Finally, Table~\ref{tab:sparsity} reports the CW errors for the graphical Gaussian distribution in Example~\ref{graphical_gaussian}.
We again use the correctly-specified \texttt{mclust} as the ideal benchmark.
Note that BAND successfully adapts to the underlying sparsity and achieves lower CW errors when there are many independent blocks among the $p = 30$ variables.
In contrast, the CW errors of \texttt{kdvine} and \texttt{nflows} remain the same regardless of the sparsity level.

{\renewcommand{\arraystretch}{0.9}
\begin{table}[h!]
\caption{Average Cram\'{e}r--Wold distance over 30 trials for the four methods. Data are generated from as in Example~\ref{graphical_gaussian}. We set $p$ = 30, $\rho = 1/2$, $n = 1{,}000$. Standard deviations are in parentheses.
}\label{tab:sparsity}
\scriptsize
\centering
\begin{tabular}{lcccc}
\toprule
\texttt{\#}\textnormal{blocks} ($m$) & 3 & 6 & 10 & 30 \\
\midrule
BAND & 0.082 (0.009) & 0.084 (0.007) & 0.077 (0.008) & 0.061 (0.007) \\
mcluster & 0.065 (0.007) & 0.064 (0.007) & 0.065 (0.008) & 0.065 (0.007) \\
nflow & 0.091 (0.011) & 0.094 (0.010) & 0.095 (0.013) & 0.095 (0.015) \\
kdvine & 0.084 (0.007) & 0.088 (0.009) & 0.087 (0.009) & 0.088 (0.007) \\\bottomrule
\end{tabular}
\end{table}}

\section{Application: Multivariate Forecasting Confidence Regions for Macroeconomic Time Series}\label{Sec5}

In this section, we apply BAND to construct (multivariate) forecasting confidence regions for U.S.~macroeconomic time series.
Macroeconomic forecasting is important but challenging, as economic time series often exhibit nonstationarity, nonlinear dynamics, and may contain substantial outliers.
We focus on three key monthly series, namely unemployment rate, federal funds rate, and inflation rate.
Before analysis, it is standard practice to remove unit root nonstationarity~\citep{mccracken2016fred}. 
We analyze $\text{UR}_t$, the first difference of the monthly U.S.~unemployment rate, $\text{FR}_t$, the first difference of the monthly effective federal funds rate, and $\text{IR}_t$, the monthly inflation rate computed as $\text{IR}_t = (\text{CPI}_{t}-\text{CPI}_{t-1})/\text{CPI}_{t-1}$, where $\text{CPI}_{t}$ is the U.S.~consumer price index. 
We consider data from February 1959 to August 2025, which are publicly available from the FRED database~\citep{mccracken2016fred}.

We consider two forecasting periods, Jan.\ 2014--Dec.\ 2019 (test period I) and Aug.\ 2022--Aug.\ 2025 (test period II) using all observations prior to each period as the corresponding training sample. 
We train BAND regression models on data $\vv{\boldsymbol{X}}_t = (\vv{\boldsymbol{U}}_t, \vv{\boldsymbol{U}}_{t-k})$
, where 
$\vv{\boldsymbol{U}}_t = (\text{FR}_t, \text{IR}_t, \text{UR}_t)$ 
and the forecasting horizon $k \in \{1,3,6\}$.
Table \ref{tab:statistics} shows some summary statistics of the data.
Note that the training data for test period II contain significant outliers in the unemployment series due to COVID-19, as shown by the extremely large range and the inflated bin-size-to-interquantile-range ratio.
The outliers could cause instability in conditional distribution estimation.
Hence we discuss test period I and II separately in our analysis.

The $k$-month-ahead forecasting confidence regions of BAND at $\alpha$ level are constructed as follows.
Suppose $\vv{\boldsymbol{U}}_{t-k}=\vv{u}_{t-k}$ is observed.
Then, focusing on the hypercubes whose last three coordinates contain $\vv{u}_{t-k}$, we select those with the highest predicted probabilities according to BAND until the cumulative probability just exceeds $\alpha$.
Next, as is standard for partition-based methods~\citep{hyndman1996computing}, the size of the final selected hypercube is adjusted by proportionally shrinking each dimension toward its center so that the resulting cumulative probability matches $\alpha$ exactly.
When constructing confidence regions for univariate variables (i.e., we are interested in marginal confidence regions), coordinates that are not of direct interest are ``integrated out'' by aggregation. 
Further details are provided in Section~\ref{SecA3} in the supplementary material.
Finally, the bin-size parameter \texttt{factor} (see Section \ref{Sec4}) is selected over a grid $\{0.4,0.5,\dots,2.0\}$ by minimizing the maximum discrepancy between the empirical coverage rate and the target $\alpha = 0.9$ level for the 6-month-ahead forecast of $\vv{\boldsymbol{U}}_{t}$.
Selected \texttt{factor} for each test period is reported in Table~\ref{tab:statistics}.

As benchmarks, we employ the quantile regression (QR) and the quantile random forests (QRF, \citealp{meinshausen2006quantile}) in our analysis.
To construct confidence regions for QR and QRF, we use the two-sided $(\frac{1-\alpha}{2}, \frac{1+\alpha}{2})$ quantile intervals, which are the default in \texttt{quantreg} and \texttt{quantregForest} in \textsf{R}. 
Multi-dimensional confidence regions are constructed by intersecting the corresponding marginal intervals, each adjusted via the Bonferroni correction: $\alpha_{\text{adj}} = 1 - \frac{1 - \alpha}{\text{Region Dimension}}$. 
Although the Bonferroni correction is known to be conservative and may be less suited for constructing multivariate forecasting confidence regions, such issue has received relatively limited attention in both classical and recent deep-learning approaches~\citep{kollovieh2023predict, kotelnikov2023tabddpm}.

{\renewcommand{\arraystretch}{0.9}
\setlength{\tabcolsep}{3pt} 
\begin{table}[ht]
\centering
\caption{
Summary statistics of the training data in the two forecasting periods, including the ranges, interquantile ranges (IQR), bin size (BS), and the BS-to-IQR (BS / IQR) ratio for each series.
The selected \texttt{factor} and the resulting number of bins (\texttt{\#}Bins) for each forecasting period are also reported.
}\label{tab:statistics}

\begin{tabular}{lccc|ccc}
\toprule
 &  \multicolumn{3}{c}{Test Period I} & \multicolumn{3}{c}{Test Period II}  \\
 &  \multicolumn{3}{c}{Jan.\ 2014--Dec.\ 2019} & \multicolumn{3}{c}{Aug.\ 2022--Aug.\ 2025}  \\
& $\text{FR}_t$ & $\text{IR}_t$ & $\text{UR}_t$
& $\text{FR}_t$ & $\text{IR}_t$ & $\text{UR}_t$ \\
\midrule
Range     & 9.69  & 0.0358 & 1.60  & 9.69 & 0.0358 & 12.60 \\
IQR       & 0.24  & 0.0033 & 0.20  & 0.20 & 0.0033 & 0.20 \\
BS        & 1.38 & 0.0051 & 0.23  & 1.07 & 0.0040 & 1.40 \\
BS / IQR  & 5.77  & 1.55   & 1.15  & 5.35 & 1.21   & 7.00 \\ \hline
\texttt{\#}Bins   & \multicolumn{3}{c|}{7}    & \multicolumn{3}{c}{9}  \\
\texttt{factor}   &  \multicolumn{3}{c|}{0.9} & \multicolumn{3}{c}{1.0} \\
\bottomrule
\end{tabular}
\end{table}}

{\renewcommand{\arraystretch}{0.9}
\begin{table}[ht]
\centering
\small
\caption{Average absolute gaps (Ave Gap), across different forecasting horizons $k \in \{1, 3, 6\}$, between the empirical coverage rate and the specified confidence level $\alpha$.
(A) average over univariate forecasts; (B) multivariate forecasts.}
\label{tab:comparison_}
\begin{tabular}{ccccc}
\toprule
Test Period & Ave Gap & Method & (A) & (B) \\
\midrule
\multirow{5}{*}{I}
 & \multirow{3}{*}{$\alpha \le 0.9$}
 &    BAND & 
          0.114 & 0.066 \\
 &  & QR   & 
          0.118 & 0.181 \\
 &  & QRF  & 
          0.114 & 0.143 \\
\cline{2-5}
 & \multirow{3}{*}{$\alpha = 0.9$}
 &    BAND & 
        0.060 & 0.032 \\
 &  & QR   & 
        0.059 & 0.061 \\
 &  & QRF  & 
        0.061 & 0.035 \\
\midrule
\multirow{5}{*}{II}
 & \multirow{3}{*}{$\alpha \le 0.9$}
 &    BAND & 
        0.186 & 0.085 \\
 &  & QR   & 
        0.115 & 0.266 \\
 &  & QRF  & 
        0.100 & 0.180 \\
\cline{2-5}
 & \multirow{3}{*}{$\alpha = 0.9$}
 &    BAND & 
        0.057 & 0.048 \\
 &  & QR   & 
        0.082 & 0.089 \\
 &  & QRF  & 
        0.047 & 0.044 \\
\bottomrule
\end{tabular}
\end{table}}

\subsection{Results}

We examine the gaps between the realized coverage rates and the target values of $\alpha$ for $\alpha \in \{0.1, 0.2, \dots, 0.9\}$. 
Table~\ref{tab:comparison_} reports the average gaps for univariate and multivariate forecasts. 
For test period I, when used to construct marginal forecasting confidence regions, the average coverage gaps of BAND for the univariate series are similar to those of QR and QRF, although it is trained using all six variables.
For multivariate forecasting confidence regions (column (B) in Table~\ref{tab:comparison_}), BAND clearly outperforms QR and QRF.
This suggests that BAND is effective for multivariate distribution estimation when trained on the full set of variables, leveraging its ability to handle higher-dimensional data.

For test period II, BAND shows larger average gaps in forecasting individual series because of the outliers. For multivariate forecast confidence regions, BAND and QRF perform comparably at $\alpha = 0.9$, while BAND achieves approximately 50\% smaller average gaps than QRF when averaged over $\alpha \leq 0.9$. Both methods outperform QR in these cases.
BAND performs more stably in multivariate than in univariate settings in period II. Intuitively, including more variables distributes the probability mass across multiple hypercubes, reducing the chance that a single hypercube dominates, which is a common issue in univariate cases and in the presence of outliers.

In summary, BAND demonstrates strong empirical performance for multivariate distribution estimation and is applicable to mixed-type distributions. By adapting to underlying sparsity, it achieves fast convergence rates that mitigate the curse of dimensionality, making it effective in high-dimensional distribution modeling.




\section*{Acknowledgements}
We thank Tengyuan Liang and Ruey Tsay for helpful comments and constructive suggestions.



\section*{Impact Statement}

This paper presents work whose goal is to advance the field of 
Machine Learning. There are many potential societal consequences 
of our work, none which we feel must be specifically highlighted here.





\bibliography{references}
\bibliographystyle{icml2026}

\newpage
\appendix
\onecolumn

\begin{center}{\bf \Large Supplementary Material to ``Breaking the Curse with BAND: Nonparametric Distribution Estimation in High Dimensions''}
\bigskip
		
\end{center}

\noindent This Supplementary Material contains the proofs of Theorem \ref{theorem2} and Corollaries \ref{corollary1} and \ref{corollary2}. Hereafter, we define $\prod_{q \in S} \boldsymbol{1}\{\cdots\} = 1$ when $S = \emptyset$, 
implying $\mathbb{P}\big(\cap_{q \in S} \{\cdots\}\big) = 1$. Let $\mathbb{R}^p_c$ and $\mathbb{R}^p(S, \vv{d}_S)$ respectively denote $\mathcal{A}_c$ and $\mathcal{A}(S, \vv{d}_S)$ with $\mathcal{A} = \mathbb{R}^p$, as in Section~\ref{Sec2.1}.

\section{Supplementary Material for Main Texts}\label{SecA}

\subsection{Empirical Cramér–Wold Distance in Section~\ref{Sec4}}\label{l1measure}

Let $\{\vv{\boldsymbol{X}}_t\}_{t=1}^{n_1}$ and $\{\vv{\boldsymbol{Y}}_t\}_{t=1}^{n_2}$ denote two samples of sizes $n_1$ and $n_2$, respectively.  
For any unit direction $\vv{v}\in\mathbb{R}^p$ with $\|\vv{v}\|_2=1$, define the empirical projected CDFs
\[
\widehat{F}_{\vv{v}}(z)
= \frac{1}{n_1}\sum_{t=1}^{n_1} 
\boldsymbol{1}\!\left\{\vv{v}^{\!\top}\vv{\boldsymbol{X}}_t \le z \right\},
\qquad
\widehat{G}_{\vv{v}}(z)
= \frac{1}{n_2}\sum_{t=1}^{n_2} 
\boldsymbol{1}\!\left\{\vv{v}^{\!\top}\vv{\boldsymbol{Y}}_t \le z \right\}.
\]
The empirical Cramér–Wold distance between the two samples is given by
\[
 \max_{z \in D}\; \max_{\vv{v} \in V}
\left|\widehat{F}_{\vv{v}}(z) - 
\widehat{G}_{\vv{v}}(z)\right|,
\]
where the set of thresholds $D$ is discretized as
$D = \{-4,-3.9,-3.8,\dots,7\}$,
to approximate the real line. The set $V$ consists of $10{,}000$ randomly generated unit vectors in $\mathbb{R}^p$, used to approximate the supremum over all directions.

In our experiments, each sequence, \(\{\vv{\boldsymbol{X}}_t\}_{t=1}^{n_1}\) or \(\{\vv{\boldsymbol{Y}}_t\}_{t=1}^{n_2}\), corresponds to one of the following: (i) a fitted distribution model, (ii) the ground-truth data-generating process, or (iii) the original training sample.

\subsection{Monte Carlo Probability Estimation and Sampling via BAND}\label{sampleing}

Here, we introduce how to use the fitted BAND model to estimate $\mathbb{P}(\boldsymbol{X} \in \mathcal{A})$ for any measurable set $\mathcal{A} \subseteq \mathbb{R}^p$. Let $N$ denote the number of samples generated from a BAND model. Let
\[
\{\widehat{\vv{h}}_{i,1}, \dots, \widehat{\vv{h}}_{i,p}\}_{i=1}^N
\]
denote the generated samples, where each $\widehat{\vv{h}}_{i,j}=(\widehat{h}_{i,j,1},\dots,\widehat{h}_{i,j,m_j})^{\top}$ is an $m_j$-dimensional one-hot vector.

For a fixed sample index $i$, the sampling procedure proceeds sequentially over coordinates $j=1,\dots,p$. For each $l\in\{1,\dots,m_j\}$, define
\[
v_l = \widehat{e}_{jl}(\widehat{\vv{h}}_{i,j-1},\dots,\widehat{\vv{h}}_{i,1}),
\]
and apply the thresholding rule
\[
\widehat{e}_{jl}^{\dagger}(\widehat{\vv{h}}_{i,j-1},\dots,\widehat{\vv{h}}_{i,1})
=
\begin{cases}
v_l, & v_l \ge \texttt{SPTF} \times \mathrm{Standard \ Deviation}\big(\{v_l\}_{l=1}^{m_j}\big), \\
0, & \text{otherwise}.
\end{cases}
\]
When $j=1$, we set $\widehat{e}_{jl}(\widehat{\vv{h}}_{i,j-1},\dots,\widehat{\vv{h}}_{i,1})=\widehat{e}_{jl}$.

The $i$th sample is then assigned to an interval index
\(
l \in \{1,\dots,m_j\}
\)
by sampling from the probability vector
\[
\left(
\frac{\widehat{e}_{j1}^{\dagger}(\widehat{\vv{h}}_{i,j-1},\dots,\widehat{\vv{h}}_{i,1})}{V},
\dots,
\frac{\widehat{e}_{jm_j}^{\dagger}(\widehat{\vv{h}}_{i,j-1},\dots,\widehat{\vv{h}}_{i,1})}{V}
\right),
\qquad
V=\sum_{l=1}^{m_j}\widehat{e}_{jl}^{\dagger}(\widehat{\vv{h}}_{i,j-1},\dots,\widehat{\vv{h}}_{i,1}).
\]
Given the sampled index $l$, we construct the one-hot vector $\widehat{\vv{h}}_{i,j}$ by setting
\[
\widehat{h}_{i,j,l} \leftarrow 1,
\qquad
\widehat{h}_{i,j,k} \leftarrow 0 \quad \text{for all } k \neq l .
\]

After completing this procedure for all coordinates, we obtain a collection of one-hot vectors
$\{\widehat{\vv{h}}_{i,l}\}_{1\le i\le N,\ 1\le l\le p}$,
where each $(\widehat{\vv{h}}_{i,1},\dots,\widehat{\vv{h}}_{i,p})$ identifies a hyper-rectangular cube. For each $i$, a point $\vv{x}_i$ is then sampled uniformly from the corresponding bin. The Monte Carlo estimator of $\mathbb{P}(\boldsymbol{X}\in\mathcal{A})$ is given by
\[
\frac{1}{N}\sum_{i=1}^{N}\boldsymbol{1}\{\vv{x}_i\in\mathcal{A}\}.
\]

The above sampling procedure depends on a tuning parameter $\texttt{SPTF}\in[0,0.4]$. Its optimal value is selected by minimizing the Cram\'{e}r--Wold distance (see Section~\ref{Sec4.2b}) between the training data and the generated samples.

\subsection{Conditional Confidence Regions Construction}\label{SecA3}

The procedure for constructing a $k$-month forecast confidence region for $\vv{\boldsymbol{U}}_t$ at confidence level $\alpha$, conditional on the past state $\vv{\boldsymbol{U}}_{t-k}$, using BAND, is as follows. First, BAND selects all hypercubes that contain the past state $\vv{\boldsymbol{U}}_{t-k}$; the predicted probabilities associated with these hypercubes sum to a total probability $\widehat{p}$. Second, BAND ranks these hypercubes by their predicted probabilities and selects them sequentially until the cumulative probability reaches $\alpha \widehat{p}$. At this point, the union of the selected hypercubes generally overestimates the desired confidence region. To correct this, the final hypercube is shrunk toward its center so that the cumulative probability matches the threshold $\alpha \widehat{p}$ exactly, under the assumption of uniform probability within each hypercube—a standard assumption for partition-based predictors such as BAND.

The resulting union of hypercubes, denoted by $\widehat{\mathcal{C}}(\vv{\boldsymbol{U}}_{t-k})$, constitutes the forecasted confidence region. The empirical coverage is then computed as 
\[
\frac{\texttt{\#}\{t \in \text{TestSet} \mid \vv{\boldsymbol{U}}_t \in \widehat{\mathcal{C}}(\vv{\boldsymbol{U}}_{t-k})\}}{\texttt{\#} \text{TestSet}}.
\]

When constructing a confidence region for a subset of target variables $\vv{\boldsymbol{V}}_t \subset \vv{\boldsymbol{U}}_t$, the irrelevant coordinates (i.e., those in $\vv{\boldsymbol{U}}_t \setminus \vv{\boldsymbol{V}}_t$) are marginalized out, and the same procedure is applied. In BAND, this corresponds to summing the predicted probabilities across all hypercubes along the irrelevant dimensions, yielding a confidence region for the variables of interest $\vv{\boldsymbol{V}}_t$.

Details of our implementation of BAND are provided in the accompanying \textsf{R} code.

\section{Proofs of Main Results}

\subsection{Theorems~\ref{theorem2} and \ref{theorem1} and Their Proofs}\label{proof.theorem1}
Theorem~\ref{theorem1} below provides the full version of Theorem~\ref{theorem2}, extending it to include results for conditional distributions while remaining otherwise identical. Let $\mathbb{R}^p_c$ and $\mathbb{R}^p(S, \vv{d}_S)$ respectively denote $\mathcal{A}_c$ and $\mathcal{A}(S, \vv{d}_S)$ with $\mathcal{A} = \mathbb{R}^p$. Note that the way we construct bins and cubes as in Section~\ref{Sec2.2} leads to 
$\mathbb{P}(\vv{\boldsymbol{X}} \in \mathcal{A} ) =  \sum_{\vv{h}_1, \dots, \vv{h}_p} \mathbb{P}(\vv{\boldsymbol{X}} \in \mathcal{A} \cap  \text{Cube}(\vv{h}_1, \dots, \vv{h}_p))$.

For the reader’s convenience, we recall the following facts from Sections~\ref{Sec2.1}--\ref{Sec2.2}:
\begin{itemize}
    \item[(1)] $\mathrm{Cube}(\vv{h}_1, \dots, \vv{h}_p)$ denotes the cube formed by the $\omega_j(\vv{h}_j)$th bin along the $j$th coordinate; all cubes are disjoint.
    
    \item[(2)] If $\vv{x} \in \mathbb{R}^p(S, \vv{d}_S)$ for some $(S, \vv{d}_S) \in \Theta$, then $\vv{x}_S$ lies in the discrete point sets $D_j$ for $j \in S$ (i.e., $\vv{x}_S = \vv{d}_S$), while each coordinate of $\vv{x}_{S^c}$ is continuous.  
    
    \item[(3)] For each $(S, \vv{d}_S) \in \Theta$, $\text{Cube}(\vv{h}_1, \dots, \vv{h}_p) \subset \mathbb{R}^p(S, \vv{d}_S)$ iff $\text{Cube}(\vv{h}_1, \dots, \vv{h}_p) \cap \mathbb{R}^p(S, \vv{d}_S) \neq \emptyset$; additionally, $\text{Cube}(\vv{h}_1, \dots, \vv{h}_p) \subset \mathbb{R}_c^p$ iff $\text{Cube}(\vv{h}_1, \dots, \vv{h}_p) \cap \mathbb{R}_c^p \neq \emptyset$.  
    
    \item[(4)] $\text{Cube}(\vv{h}_1, \dots, \vv{h}_p) \subset [-\delta_n, \delta_n]^p$ iff $\text{Cube}(\vv{h}_1, \dots, \vv{h}_p) \cap [-\delta_n, \delta_n]^p \neq \emptyset$.
    
    \item[(5)] $\pi_{S, \vv{d}_S}(\vv{x}) = 1$ for $\vv{x}$ in a discrete cube (i.e., $S^c = \emptyset$), and $\pi_c(\vv{x}) = 0$ when $\mathbb{R}_c^p = \emptyset$.  
\end{itemize}

\begin{theorem}\label{theorem1} 
Let \( \delta_n > 0 \) be a positive sequence. Set \(\widehat{e}_{jl}(\vv{h}_{j-1}, \dots, \vv{h}_1) = 0\) 
if \(\omega_q(\vv{h}_q) = m_q\) for some \(1 \le q < j\) or \(l = m_j\).  For all large $n$, assume Condition~\ref{mean.regression}, $\frac{ep\sqrt{T}}{d_{\min}} \le 1$, $\varepsilon  \le (\sqrt{p}\delta_n K_{1})^{-1}$, and that the $p_n$-dimensional stationary distribution of $\vv{\boldsymbol{X}}_t$ satisfies Condition~\ref{conti} with $R = \delta_n$. It holds for all large \(n\) that with probability at least \(1 - \Delta\), for all \(Q \subset \{1, \dots, p\}\) with \(1 \le \texttt{\#}Q < p\), 
\(\vartheta > 1\), and \(\mathcal{H} \in \mathcal{R}^{Q^c}\) satisfying 
\(\mathbb{P}(\vv{\boldsymbol{X}}_{Q^c} \in \mathcal{H}) \ge \vartheta \, \zeta\):
\begin{equation*}
\begin{split}
 & \sup_{\mathcal{A} \in \mathcal{R}^p}\big| \mathbb{P}(\vv{\boldsymbol{X}} \in \mathcal{A}) - \widehat{\mu}(\mathcal{A}) \big|  \le  
       \frac{ep\sqrt{T}}{d_{\min}}  
        + \varepsilon   \sqrt{p} \delta_n K_{1} e + \mathbb{P}(\vv{\boldsymbol{X}}\not\in [-\delta_n, \delta_n]^p )\eqqcolon \zeta,
\end{split}
\end{equation*}
and that
\begin{equation*}
    \begin{split}        
 & \sup_{\mathcal{B} \in \mathcal{R}^{\texttt{\#}Q}} \left| \mathbb{P}(\vv{\boldsymbol{X}}_Q \in \mathcal{B} \mid \vv{\boldsymbol{X}}_{Q^c} \in \mathcal{H}) - \frac{\widehat{\mu}(\{\vv{x}\in\mathbb{R}^p\mid \vv{x}_Q \in \mathcal{B} ,  \vv{x}_{Q^c} \in \mathcal{H}\} )} {\widehat{\mu}(\{\vv{x}\in\mathbb{R}^p\mid \vv{x}_{Q^c} \in \mathcal{H}\} )} \right|\le \frac{2}{(\vartheta-1)}.
    \end{split}
\end{equation*}

\end{theorem}

Note that $(\Delta, d_{\min}, T, E^{\dagger}, \varepsilon, p) = (\Delta_n, d_{\min, n}, T_n, E_{n}^{\dagger}, \varepsilon_n, p_n)$ but the subscripts $n$ are omitted for simplicity.

\noindent \textit{Proof of Theorem~\ref{theorem1}: }  We begin with defining an auxiliary distribution \(\widetilde{\mu}\) 
in \eqref{prop2.15}, which is the population counterpart of our distribution estimator~\eqref{estimate.stationary.dist}. For each $(\vv{h}_1, \dots, \vv{h}_p)\in \{0, 1\}^{m_1}\times \dots \times \{0, 1\}^{m_p}$ with one-hot $\vv{h}_j$,
every measurable $\mathcal{A}\subset \mathbb{R}^p$, and each $\vv{x}\in \text{Cube}(\vv{h}_1, \dots, \vv{h}_p)$, define
\begin{equation}
    \begin{split}\label{prop2.15}
\widetilde{\mu}( \mathcal{A}) & = \sum_{\vv{h}_1, \dots, \vv{h}_p} \frac{\bar{\nu}_{\vv{h}_{1:p}}\big(\mathcal{A} \cap \text{Cube}(\vv{h}_1, \dots, \vv{h}_p)\big)}{\bar{\nu}_{\vv{h}_{1:p}}\big(\text{Cube}(\vv{h}_1, \dots, \vv{h}_p) \big) }
\prod_{j=1}^p e_{j,\omega_j(h_j)}(\vv{h}_{j-1}, \dots, \vv{h}_{1}) , 
\end{split}
\end{equation}
where $\bar{\nu}_{\vv{h}_{1:p}}$ denotes Lebesgue measure with respect to the continuous coordinates of $\text{Cube}(\vv{h}_{1:p})$, with zero measure outside the cube. In addition, if $\text{Cube}(\vv{h}_{1:p})$ is a discrete cube contained by $\mathcal{G}$, then $\bar{\nu}_{\vv{h}_{1:p}}(\mathcal{G}) = 1$; if $\text{Cube}(\vv{h}_{1:p})$ is a discrete cube not contained by $\mathcal{G}$, then $\bar{\nu}_{\vv{h}_{1:p}}(\mathcal{G}) = 0$. 

In what follows, we clarify the definitions of $\bar{\nu}_{\vv{h}_{1:p}}$ and
$\widetilde{\mu}$, which will be used repeatedly in the sequel. In particular, $\bar{\nu}_{\vv{h}_{1:p}}$ admits a direct integral representation.
By fact (3) stated at the beginning of Section~\ref{proof.theorem1}, for any
$\mathcal{A} \in \mathcal{R}^p$, $\bar{\nu}_{\vv{h}_{1:p}}(\mathcal{A})$ satisfies
\begin{equation}
\label{ientity.2}
\begin{aligned}
\bar{\nu}_{\vv{h}_{1:p}}(\mathcal{A}) &= 0,
&& \text{if } \mathcal{A} \cap \mathrm{Cube}(\vv{h}_1, \dots, \vv{h}_p) = \emptyset, \\[0.6em]
\bar{\nu}_{\vv{h}_{1:p}}(\mathcal{A})
&= \bar{\nu}_{\vv{h}_{1:p}}\big( \mathcal{A}(S, \vv{d}_S) \big)
= \int_{\vv{u} \in \mathcal{A}(S, \vv{d}_S) \cap \mathrm{Cube}(\vv{h}_1, \dots, \vv{h}_p)}
1 \, d\vv{u}_{S^c},
&& \text{if } \mathcal{A}(S, \vv{d}_S) \cap \mathrm{Cube}(\vv{h}_1, \dots, \vv{h}_p) \neq \emptyset, \\[0.8em]
\bar{\nu}_{\vv{h}_{1:p}}(\mathcal{A})
&= \bar{\nu}_{\vv{h}_{1:p}}(\mathcal{A}_c)
= \int_{\vv{u} \in \mathcal{A}_c \cap \mathrm{Cube}(\vv{h}_1, \dots, \vv{h}_p)}
1 \, d\vv{u},
&& \text{if } \mathcal{A}_c \cap \mathrm{Cube}(\vv{h}_1, \dots, \vv{h}_p) \neq \emptyset .
\end{aligned}
\end{equation}
Hereafter, when $S^c \neq \emptyset$, integrals with respect to $d\vv{u}_{S^c}$ are understood to be taken over the coordinates indexed by $S^c$, with the remaining coordinates indexed by $S$ fixed at $\vv{d}_S$; when $S^c = \emptyset$, such integrals are defined to equal $1$. Under this convention, the second integral may be written as $\int_{\vv{u} \in \bar{\mathcal{A}}_{\vv{h}_{1:p}}(S, \vv{d}_S)} 1 \, d\vv{u}_{S^c}$, where $\bar{\mathcal{A}}_{\vv{h}_{1:p}}(S, \vv{d}_S) = \{\vv{a}_{S^c} : \vv{a} \in \mathcal{A}(S, \vv{d}_S) \cap \mathrm{Cube}(\vv{h}_1, \dots, \vv{h}_p)\}$. The definition of \eqref{ientity.2} ensures that all integrals appearing in the subsequent derivations are expressed in a consistent and unified manner.

Moreover, by the definition of $\nu_{\vv{h}_{1:p}}$ in \eqref{estimate.stationary.dist},
\[
\nu_{\vv{h}_{1:p}}(\mathcal{A})
= \frac{\bar{\nu}_{\vv{h}_{1:p}}\big( \mathcal{A} \cap \mathrm{Cube}(\vv{h}_1, \dots, \vv{h}_p) \big)}
{\bar{\nu}_{\vv{h}_{1:p}}\big( \mathrm{Cube}(\vv{h}_1, \dots, \vv{h}_p) \big)}
= \frac{\bar{\nu}_{\vv{h}_{1:p}}(\mathcal{A})}
{\bar{\nu}_{\vv{h}_{1:p}}\big( \mathrm{Cube}(\vv{h}_1, \dots, \vv{h}_p) \big)}.
\]

Finally, by the cube construction (following \eqref{estimate.stationary.dist})
and the definition of $\widetilde{\mu}$ in \eqref{prop2.15}, both
$\mathbb{P}(\vv{\boldsymbol{X}} \in \cdot)$ and $\widetilde{\mu}$ decompose over the
partition: for any $\mathcal{A} \in \mathcal{R}^p$,
\begin{equation}
\label{cube.partiion}
\begin{split}
\mathbb{P}(\vv{\boldsymbol{X}} \in \mathcal{A})
&= \sum_{\vv{h}_1, \dots, \vv{h}_p}
\mathbb{P}\big( \vv{\boldsymbol{X}} \in \mathcal{A} \cap \mathrm{Cube}(\vv{h}_1, \dots, \vv{h}_p) \big), \\
\widetilde{\mu}(\mathcal{A})
&= \sum_{\vv{h}_1, \dots, \vv{h}_p}
\widetilde{\mu}\big( \mathcal{A} \cap \mathrm{Cube}(\vv{h}_1, \dots, \vv{h}_p) \big),
\end{split}
\end{equation}
where the sum ranges over all one-hot vectors.

A bias–variance decomposition of the \(L_1\) upper bound for our distribution estimator is derived as follows. For any measurable set \(\mathcal{A} \subset \mathbb{R}^p\),
\begin{equation}
\label{theorem1.31}
\big|\mathbb{P}(\vv{\boldsymbol{X}} \in \mathcal{A}) - \widehat{\mu}(\mathcal{A})\big|
\le \big|\widetilde{\mu}(\mathcal{A}) - \widehat{\mu}(\mathcal{A})\big|
+ \big|\mathbb{P}(\vv{\boldsymbol{X}} \in \mathcal{A}) - \widetilde{\mu}(\mathcal{A})\big|.
\end{equation}
We briefly outline the strategy used to analyze \eqref{theorem1.31}. The second term corresponds to the bias. To illustrate how this term is controlled by the estimator defined in \eqref{prop2.15}, we first consider the continuous component while temporarily ignoring tail probabilities. In this case, \eqref{prop2.15} implies that the density \(\pi_c(\vv{x})\) is approximated on each cube by \(\mathbb{P}(\vv{\boldsymbol{X}} \in \mathrm{Cube}(\vv{h}_1, \dots, \vv{h}_p))\), as shown in \eqref{prop2.6} and \eqref{prop2.7} below. When the maximum cube diameter \(\varepsilon\) is sufficiently small, this approximation is accurate provided that \(\pi_c\) satisfies the smoothness Condition~\ref{conti}. The first term represents the variance. Its control relies on Condition~\ref{mean.regression}, which is used to bound the discrepancy between \(\prod_{j=1}^p e_{j, \omega_j(h_j)}(\vv{h}_{j-1}, \dots, \vv{h}_1)\) and its empirical counterpart \(\prod_{j=1}^p \widehat{e}_{j, \omega_j(h_j)}(\vv{h}_{j-1}, \dots, \vv{h}_1)\), as established in \eqref{prop2.17} below.


We now establish an upper bound for the second term on the right-hand side of \eqref{theorem1.31}. Analogously to the cube partition argument in \eqref{cube.partiion}, we derive a probability upper bound by partitioning according to the discrete–continuous structure of the distribution. In particular, for every measurable set $\mathcal{A} \subset [-\delta_n, \delta_n]^p$, we obtain the following partition inequality.
\begin{equation}
    \begin{split}\label{prop2.16.bb}
          & |\mathbb{P}(\vv{\boldsymbol{X}} \in \mathcal{A} ) -  \widetilde{\mu}(\mathcal{A})| \\
        & \le |\mathbb{P}(\vv{\boldsymbol{X}} \in \mathcal{A}_c ) -  \widetilde{\mu}(\mathcal{A}_c)| + \sum_{(S, \vv{d}_S) \in\Theta} |\mathbb{P}(\vv{\boldsymbol{X}} \in \mathcal{A}(S, \vv{d}_S) ) -  \widetilde{\mu}(\mathcal{A}(S, \vv{d}_S) )|.
    \end{split}    
\end{equation}
We will bound each deviation appearing on the right-hand side of \eqref{prop2.16.bb} individually. To this end, we compare the quantity $\prod_{j=1}^p e_{j, \omega_j(h_j)}(\vv{h}_{j-1}, \dots, \vv{h}_1)$ with the corresponding underlying conditional density on each hypercube, as formalized in \eqref{prop2.8} below.

An expression for $\prod_{j=1}^p e_{j, \omega_j(h_j)}(\vv{h}_{j-1}, \dots, \vv{h}_1)$ is derived as follows. By \eqref{stationary.dist}–\eqref{band}, together with the definitions of the random vectors $\vv{\boldsymbol{H}}_j$ and the probability functions $e_{jl}$ given in Section~\ref{Sec2.2}, it follows that for every one-hot vector $\vv{h}_j \in \{0,1\}^{m_j}$,
\begin{equation}
\label{prop2.6}
\begin{split}  
& \prod_{j=1}^p e_{j, \omega_j(h_j)}(\vv{h}_{j-1}, \dots, \vv{h}_1) \\
&= \mathbb{P}\!\left(\bigcap_{j=1}^p \{\vv{\boldsymbol{H}}_j = \vv{h}_j\}\right) \\
&= \mathbb{P}\!\left(\vv{\boldsymbol{X}} \in \text{Cube}(\vv{h}_1, \dots, \vv{h}_p)\right) \\
&= 
\begin{cases}
\mathbb{P}(\vv{\boldsymbol{X}}_S = \vv{d}_S ) \displaystyle\int_{\vv{u}\in \text{Cube}(\vv{h}_1, \dots, \vv{h}_p)} \pi_{S, \vv{d}_S}(\vv{u}) \, d\vv{u}_{S^c}, & \text{if } \text{Cube}(\vv{h}_1, \dots, \vv{h}_p)\cap \mathbb{R}^p(S, \vv{d}_S) \not=\emptyset,
\\[4ex]
\displaystyle\int_{\vv{u}\in \text{Cube}(\vv{h}_1, \dots, \vv{h}_p)} \pi_c(\vv{u}) \, d\vv{u}, & \text{if } \text{Cube}(\vv{h}_1, \dots, \vv{h}_p)\cap  \mathbb{R}^p_c \not=\emptyset.
\end{cases}
\end{split}
\end{equation}
Here, the third equality follows from fact (3) stated at the beginning of Section~\ref{proof.theorem1}. We also recall the definition of integration with respect to $d\vv{u}_{S^c}$ from \eqref{ientity.2}.

We next control the discrepancy between these expressions and their approximating densities. By the definition of $\varepsilon$, Lemma~\ref{lemma11} in Section~\ref{proof.lemma11}, Condition~\ref{conti} with $R=\delta_n$, and the assumption $\varepsilon \le (K_1 \sqrt{p}\delta_n)^{-1}$, we obtain that every $\pi_{\theta}$ with $\theta \in \Theta^{\star}$ (see Section~\ref{Sec4.2}) satisfies
\begin{equation}
    \label{conti2}
    \sup_{\norm{\vv{y}}_2 \le \delta_n, \norm{\vv{y} - \vv{x}}_2 \le \varepsilon } |\pi_{\theta}(\vv{y}) - \pi_{\theta}(\vv{x})| \le K_{1}\sqrt{p} \varepsilon  \times \delta_n \times e \times \pi_{\theta}(\vv{x})
\end{equation}
on $\norm{\vv{x}}_{\infty}\le  \delta_n$. Combining \eqref{conti2} with the fact that \(\varepsilon\) denotes the maximum diameter of cubes within \( [-\delta_n, \delta_n]^p\), and applying the Mean Value Theorem for Integrals, it follows that for each $\vv{x}\in [-\delta_n, \delta_n]^p$,
\begin{equation}
\label{prop2.8}
\begin{aligned}
\frac{|\mathbb{P}(\vv{\boldsymbol{X}}_S = \vv{d}_S )\,\pi_{S, \vv{d}_S}(\vv{x}) - \widetilde{\pi}(\vv{x})|}{ \mathbb{P}(\vv{\boldsymbol{X}}_S = \vv{d}_S )}
&\le  \varepsilon   \sqrt{p} \delta_n K_{1} e \pi_{S, \vv{d}_S}(\vv{x}), && \text{if } \vv{x} \in \text{Cube}(\vv{h}_1, \dots, \vv{h}_p)\cap \mathbb{R}^p(S, \vv{d}_S),\\
|\pi_c(\vv{x}) - \widetilde{\pi}(\vv{x})|
&\le \varepsilon   \sqrt{p} \delta_n K_{1} e \pi_{c}(\vv{x}) ,             &&   \text{if } \vv{x} \in \text{Cube}(\vv{h}_1, \dots, \vv{h}_p)\cap \mathbb{R}^p_c,
\end{aligned}
\end{equation}
which follows from \eqref{ientity.2} and the definition that for every $\vv{x}\in [-\delta_n, \delta_n]^p $,
\begin{equation}
\label{prop2.7}
\widetilde{\pi}(\vv{x}) = 
\begin{cases}
\displaystyle
\frac{\mathbb{P}(\vv{\boldsymbol{X}}_S = \vv{d}_S ) \int_{\vv{u}\in \text{Cube}(\vv{h}_1, \dots, \vv{h}_p)} 
\pi_{S, \vv{d}_S}(\vv{u}) \, d\vv{u}_{S^c}}
{\bar{\nu}_{\vv{h}_{1:p}}\big( \text{Cube}(\vv{h}_1, \dots, \vv{h}_p)\big)},
& \text{if } \vv{x} \in \text{Cube}(\vv{h}_1, \dots, \vv{h}_p)\cap \mathbb{R}^p(S, \vv{d}_S), \\[3ex]
\displaystyle
\frac{\int_{\vv{u}\in \text{Cube}(\vv{h}_1, \dots, \vv{h}_p)} 
\pi_c(\vv{u}) \, d\vv{u}}
{\bar{\nu}_{\vv{h}_{1:p}}\big( \text{Cube}(\vv{h}_1, \dots, \vv{h}_p)\big)},
& \text{if } \vv{x} \in \text{Cube}(\vv{h}_1, \dots, \vv{h}_p)\cap \mathbb{R}^p_c .
\end{cases}
\end{equation}

 Note that we have defined $\pi_{c}(\vv{x}) = 0$ for discrete distributions; hence, 
$|\pi_c(\vv{x}) - \widetilde{\pi}(\vv{x})| = 0$. 
Moreover, since $\pi_{S, \vv{d}_S}(\vv{x}) = 1$ when $\mathbb{R}^p(S, \vv{d}_S)$ consists of a single discrete point, it follows that 
$|\mathbb{P}(\vv{\boldsymbol{X}}_S = \vv{d}_S)\,\pi_{S, \vv{d}_S}(\vv{x}) - \widetilde{\pi}(\vv{x})| = 0$. 
Therefore, \eqref{prop2.8} may not be tight in this case, but it is sufficient for our analysis.

We now complete the derivation of upper bounds for the terms appearing on the right-hand side of \eqref{prop2.16.bb}. Combining \eqref{prop2.15}–\eqref{cube.partiion}, \eqref{prop2.6}, and \eqref{prop2.8}, together with fact (3) stated at the beginning of Section~\ref{proof.theorem1}, we obtain that for every measurable set $\mathcal{A} \subset [-\delta_n, \delta_n]^p$,
\begin{equation*}
    \begin{split}
        & |\mathbb{P}(\vv{\boldsymbol{X}} \in \mathcal{A}_c ) -  \widetilde{\mu}(\mathcal{A}_c)| \\
        & \le \sum_{\vv{h}_1, \dots, \vv{h}_p} \left|\mathbb{P}(\vv{\boldsymbol{X}} \in \mathcal{A}_c \cap  \text{Cube}(\vv{h}_1, \dots, \vv{h}_p)) -  \widetilde{\mu}(\mathcal{A}_c\cap  \text{Cube}(\vv{h}_1, \dots, \vv{h}_p)) \right| \\
        & = \sum_{\vv{h}_1, \dots, \vv{h}_p} \bigg|\int_{\vv{u}\in \mathcal{A}_c\cap \text{Cube}(\vv{h}_1, \dots, \vv{h}_p)} \pi_c(\vv{u}) \, d\vv{u} \\
        &\quad\quad\quad - 
\frac{\bar{\nu}_{\vv{h}_{1:p}}\big(\mathcal{A}_c \big)}{\bar{\nu}_{\vv{h}_{1:p}}\big(\text{Cube}(\vv{h}_1, \dots, \vv{h}_p)\big)} \int_{\vv{u}\in \text{Cube}(\vv{h}_1, \dots, \vv{h}_p)} \pi_c(\vv{u}) \, d\vv{u} \bigg|\\
& \le \sum_{\vv{h}_1, \dots, \vv{h}_p} \int_{\vv{u}\in \mathcal{A}_c\cap \text{Cube}(\vv{h}_1, \dots, \vv{h}_p)}  |\pi_c(\vv{u}) - \widetilde{\pi}(\vv{u}) | \, d\vv{u}  \\
& \le \varepsilon   \sqrt{p} \delta_n K_{1} e \times \int_{\vv{x}\in \mathbb{R}_c^{p}} \pi_c(\vv{x})d\vv{x} .
    \end{split}    
\end{equation*}
Similarly, for each $(S, \vv{d}_S) \in \Theta$,
\begin{equation*}
    \begin{split}
        & |\mathbb{P}(\vv{\boldsymbol{X}} \in \mathcal{A}(S, \vv{d}_S) ) -  \widetilde{\mu}(\mathcal{A}(S, \vv{d}_S))|  \\
        & \le \sum_{\vv{h}_1, \dots, \vv{h}_p} \bigg| \mathbb{P}(\vv{\boldsymbol{X}}_{S} = \vv{d}_S) \int_{\vv{u}\in \mathcal{A}(S, \vv{d}_S) \ \cap \ \text{Cube}(\vv{h}_1, \dots, \vv{h}_p)} \pi_{S, \vv{d}_S}(\vv{u}) \, d\vv{u}_{S^c} \\
        &\quad\quad\quad - 
\frac{\bar{\nu}_{\vv{h}_{1:p}}\big(\mathcal{A}(S, \vv{d}_S)  \big)}{\bar{\nu}_{\vv{h}_{1:p}}\big(\text{Cube}(\vv{h}_1, \dots, \vv{h}_p)\big)} \times \mathbb{P}(\vv{\boldsymbol{X}}_{S} = \vv{d}_S) \times  \int_{\vv{u}\in \text{Cube}(\vv{h}_1, \dots, \vv{h}_p)} \pi_{S, \vv{d}_S}(\vv{u}) \, d\vv{u}_{S^c} \bigg|\\
& \le \sum_{\vv{h}_1, \dots, \vv{h}_p}  \int_{\vv{u}\in \mathcal{A}(S, \vv{d}_S) \ \cap \ \text{Cube}(\vv{h}_1, \dots, \vv{h}_p)} \left| \mathbb{P}(\vv{\boldsymbol{X}}_{S} = \vv{d}_S) \times  \pi_{S, \vv{d}_S}(\vv{u}) - \widetilde{\pi}(\vv{u}) \right| d\vv{u}_{S^c} \\
        & \le \varepsilon   \sqrt{p} \delta_n K_{1} e \times \mathbb{P}(\vv{\boldsymbol{X}}_{S} = \vv{d}_S)\times\int_{\vv{x}\in \mathbb{R}^{p}(S, \vv{d}_S)} \pi_{S, \vv{d}_S}(\vv{x})d\vv{x}_{S^c}.
    \end{split}    
\end{equation*}

 With these and \eqref{prop2.16.bb}, we derive that for every measurable $\mathcal{A} \subset [-\delta_n, \delta_n]^p$, 
\begin{equation}
    \begin{split}\label{prop2.16.b}
          & |\mathbb{P}(\vv{\boldsymbol{X}} \in \mathcal{A} ) -  \widetilde{\mu}(\mathcal{A})| \\          
          &\le  \varepsilon   \sqrt{p} \delta_n K_{1} e \times (\int_{\vv{x}\in \mathbb{R}_c^{p}} \pi_c(\vv{x})d\vv{x} + \sum_{(S, \vv{d}_S) \in\Theta}\mathbb{P}(\vv{\boldsymbol{X}}_{S} = \vv{d}_S)\times\int_{\vv{x}\in \mathbb{R}^{p}(S, \vv{d}_S)} \pi_{S, \vv{d}_S}(\vv{x})d\vv{x}_{S^c})\\
          &\le \varepsilon   \sqrt{p} \delta_n K_{1} e ,
    \end{split}    
\end{equation}
where the third inequality follows from \eqref{stationary.dist}. This establishes an upper bound on the second term on the RHS of \eqref{theorem1.31}.

Now, we turn to establish an upper bound on the first term on the RHS of \eqref{theorem1.31}. By Condition~\ref{mean.regression} and the definitions of $\widehat{\mu}$ and $\widetilde{\mu}$ respectively in \eqref{estimate.stationary.dist} and \eqref{prop2.15}, it holds on $E^{\dagger}$ (given in Condition~\ref{mean.regression}), for every measurable $\mathcal{A} \subset [-\delta_n, \delta_n]^p$,
\begin{equation}
    \begin{split}\label{prop2.17}
        &\left|\widetilde{\mu}(\mathcal{A}) -  \widehat{\mu}(\mathcal{A}) \right|\\               
        & \le \sum_{ \omega_q(\vv{h}_{q}) <  m_q \text{ for } 1\le  q\le p } \left|\prod_{j=1}^p e_{j,\omega_j(\vv{h}_{j})}(\vv{h}_{j-1}, \dots, \vv{h}_{1}) - \prod_{j=1}^p \widehat{e}_{j,\omega_j(\vv{h}_{j})}(\vv{h}_{j-1}, \dots, \vv{h}_{1}) \right|\\
        & \le \sum_{\omega_q(\vv{h}_{q}) < m_q \text{ for } 1\le q\le p} \bigg(\prod_{j=1}^p e_{j,\omega_j(\vv{h}_j)}(\vv{h}_{j-1}, \dots, \vv{h}_{1}) \\  
        & \quad\quad \times \sum_{j=1}^p \frac{\sqrt{T}}{d_{\min}} 
         \times \max_{1\le k\le p} \,\max_{\omega_q(\vv{z}_{q}) < m_q \text{ for } 1\le q\le k}\left(\frac{\widehat{e}_{k,\omega_k(\vv{z}_k)}(\vv{z}_{k-1}, \dots, \vv{z}_{1})}{e_{k,\omega_k(\vv{z}_k)}(\vv{z}_{k-1}, \dots, \vv{z}_{1})} \vee 1\right)^{p} \bigg)\\
        & \le \frac{p\sqrt{T}}{d_{\min}} \times \left(1 + \frac{ \sqrt{T}}{ d_{\min}}\right)^p\\
        & \le \frac{ep\sqrt{T}}{d_{\min}},
    \end{split}
\end{equation}
where \( \omega_j(\vv{h}_j) = q \) if the \( q \)th coordinate of \( \vv{h}_j \) equals one, 
and the summation is taken over all 
\((\vv{h}_1, \dots, \vv{h}_p) \in \{0,1\}^{m_1} \times \cdots \times \{0,1\}^{m_p}\),
with each \(\vv{h}_j\) being one-hot and satisfying the specified requirements. Here, we apply the identity for positive $a_i$'s and $b_i$'s:
\begin{equation}
    \begin{split}
        \prod_{i=1}^p a_i - \prod_{i=1}^p b_i 
& = \sum_{j=1}^p (a_j - b_j)
\Biggl(\prod_{i=1}^{j-1} b_i\Biggr)
\Biggl(\prod_{i=j+1}^p a_i\Biggr)\\
& = \Biggl(\prod_{i=1}^p a_i\Biggr) \times \sum_{j=1}^p \frac{(a_j - b_j)}{a_j}
\Biggl(\prod_{i=1}^{j-1} \frac{b_i}{a_i}\Biggr),
    \end{split}
\end{equation}
along with Condition~\ref{mean.regression} and the assumption that
\begin{equation}
    \label{corollary1.3}
    \min_{j, l < m_j, \omega_q(\vv{h}_{q}) \not = m_q \text{ for } 1\le q< j }
    e_{j,l}(\vv{h}_{j-1}, \dots, \vv{h}_{1})
    \ge d_{\min},
\end{equation}
to derive the second inequality.  
The third inequality follows from Condition~\ref{mean.regression},
$$\sum_{\vv{h}_1, \dots, \vv{h}_p}
 \prod_{j=1}^p e_{j,\omega_j(h_j)}(\vv{h}_{j-1}, \dots, \vv{h}_1)
 = 1,$$
 the assumption~\eqref{corollary1.3}, and the assumption that \(\widehat{e}_{jl}(\vv{h}_{j-1}, \dots, \vv{h}_1) = 0\) 
whenever \(\omega_q(\vv{h}_q) = m_q\) for some \(1 \le q < j\) or \(l = m_j\).  
Finally, we use
\[
\Bigl(1+\frac{\sqrt{T}}{d_{\min}}\Bigr)^p
= \exp\Bigl(p\log_e \big(1+\tfrac{\sqrt{T}}{d_{\min}}\big)\Bigr)
\le \exp \Bigl(\tfrac{p\sqrt{T}}{d_{\min}}\Bigr)
\le e,
\]
when \(\tfrac{p\sqrt{T}}{d_{\min}} \le 1\), to obtain the fourth inequality.

By \eqref{theorem1.31}, \eqref{prop2.16.b}--\eqref{prop2.17}, and the assumption $\sqrt{p}\delta_n \varepsilon K_{1} \le 1$, we derive that for every measurable $\mathcal{A} \subset [-\delta_n, \delta_n]^p$, on $E^{\dagger}$,
\begin{equation}\label{prop2.16}    
    |\mathbb{P}(\vv{\boldsymbol{X}} \in \mathcal{A} ) -  \widehat{\mu}(\mathcal{A})| \le \varepsilon   \sqrt{p} \delta_n K_{1} e +  \frac{ep\sqrt{T}}{d_{\min}}.
\end{equation}

By the definitions of \eqref{estimate.stationary.dist} and the assumption that \(\widehat{e}_{jl}(\vv{h}_{j-1}, \dots, \vv{h}_1) = 0\) 
whenever \(\omega_q(\vv{h}_q) = m_q\) for some \(1 \le q < j\) or \(l = m_j\), we have that $\widehat{\mu}(\mathcal{A}) = \widehat{\mu}(\mathcal{A}\cap [-\delta_n, \delta_n]^{p})$. By this, Condition~\ref{mean.regression}, and \eqref{prop2.16}, we deduce that, with probability at least \( 1 - \Delta \) (on \( E^{\dagger} \)), the following holds for all large $n$ and every \( \mathcal{A} \in \mathcal{R}^p \).
\begin{equation}
    \begin{split}
        & | \mathbb{P}(\vv{\boldsymbol{X}} \in \mathcal{A}) - \widehat{\mu}(\mathcal{A})| \\
        & \le  \mathbb{P}(\vv{\boldsymbol{X}} \in \mathcal{A}\backslash [-\delta_n, \delta_n]^{p} ) +  \left| \mathbb{P}(\vv{\boldsymbol{X}} \in \mathcal{A}\cap [-\delta_n, \delta_n]^{p} ) - \widehat{\mu}(\mathcal{A}) \right| \\
        & =  \mathbb{P}(\vv{\boldsymbol{X}} \in \mathcal{A}\backslash [-\delta_n, \delta_n]^{p} ) +  \left| \mathbb{P}(\vv{\boldsymbol{X}} \in \mathcal{A}\cap [-\delta_n, \delta_n]^{p} ) - \widehat{\mu}(\mathcal{A}\cap [-\delta_n, \delta_n]^{p}) \right| \\
        & \le \mathbb{P}(\vv{\boldsymbol{X}}\not\in [-\delta_n, \delta_n]^p ) + \varepsilon   \sqrt{p} \delta_n K_{1} e + \frac{ep\sqrt{T}}{d_{\min}}  .
    \end{split}
\end{equation}

This completes the proof of the first assertion of Theorem~\ref{theorem1}. We now turn to establishing the upper bound for the conditional probability deviation.

Recall that 
$$\zeta = \frac{ep\sqrt{T}}{d_{\min}}
        + \varepsilon   \sqrt{p} \delta_n K_{1} e + \mathbb{P}(\vv{\boldsymbol{X}}\not\in [-\delta_n, \delta_n]^p ).$$
For any $\vartheta > 1$, we have the following result. Assume 
$a,b  > 0$ and $x,y \in (0,1]$,
with $y \ge \vartheta \zeta > \zeta$, $x \le y$, $a \le b$, and 
$|x-a| \le \zeta$, $|y-b| \le \zeta$.
Then
\begin{equation}
\begin{split}\label{prop2.18}
\left|\frac{x}{y} - \frac{a}{b}\right|
&= \left|\frac{x(b-y) + y(x-a)}{yb}\right| \\[6pt]
&\le \frac{|x||b-y| + |y||x-a|}{|y\,b|} \\[6pt]
&\le \frac{|x|\zeta + |y|\zeta}{|y|\,|b|} 
 \\[6pt]
&\le \frac{|x|\zeta + |y|\zeta}{|y|\,(y - |y-b|)} 
 \\[6pt]
&\le \frac{y\zeta + y\zeta}{y(y-\zeta)} 
\\[6pt]
& \le \frac{2}{\vartheta - 1}.
\end{split}
\end{equation}

Now, applying the first result of Theorem~\ref{theorem1} and \eqref{prop2.18} with  $x = \mathbb{P}(\vv{\boldsymbol{X}}_Q \in \mathcal{B}, \vv{\boldsymbol{X}}_{Q^c} \in \mathcal{H})$, $y = \mathbb{P}(\vv{\boldsymbol{X}}_{Q^c} \in \mathcal{H})$, $a = \widehat{\mu}(\{\vv{x}\in\mathbb{R}^p \mid \vv{x}_Q \in \mathcal{B} ,  \vv{x}_{Q^c} \in \mathcal{H}\} )$, and $b = \widehat{\mu}(\{\vv{x}\in\mathbb{R}^p \mid \vv{x}_{Q^c} \in \mathcal{H}\} )$, we obtain the desired result for the second assertion of Theorem~\ref{theorem1}, thereby completing its proof.

\subsection{Proof of Corollary~\ref{corollary1}}\label{SecA.2}

We show that the assumptions of Theorem~\ref{theorem1} are satisfied under the setting of Corollary~\ref{corollary1}. In particular, we verify Condition~\ref{mean.regression} for this setting and then apply Theorem~\ref{theorem1} to conclude the result. We begin by deriving a positive lower bound on $d_{\min}>0$ as required by Condition~\ref{mean.regression}.

  By the definition of $e_{jl}$ in \eqref{band} and Condition~\ref{sparsity}(a) or \ref{sparsity}(b), it holds that
\begin{equation}
\begin{split}\label{corollary1.5}
    e_{jl}(\vv{h}_{j-1}, \dots, \vv{h}_{1} ) & = \mathbb{E}[ \boldsymbol{H}_{jl} \mid \vv{\boldsymbol{H}}_{j-1} = \vv{h}_{j-1}, \dots, \vv{\boldsymbol{H}}_{1} = \vv{h}_{1}]
    \\
    & = \mathbb{E}[ \boldsymbol{H}_{jl} \mid \vv{\boldsymbol{H}}_q = \vv{h}_q, q\in S_{j}]\\
    & = \frac{\mathbb{E}[\boldsymbol{H}_{jl} \times \prod_{ q\in S_{j}}\boldsymbol{1}\{\vv{\boldsymbol{H}}_{q} = \vv{h}_q \} ]}{\mathbb{P}(\cap_{ q\in S_{j}}\{\vv{\boldsymbol{H}}_{q} = \vv{h}_q \}) }.
    \end{split}
\end{equation}
Notation follows Section~\ref{Sec2.2}, with $\mathbb{E}[\cdots \mid \emptyset] = \mathbb{E}(\cdots)$, $\prod_{j \in \emptyset} \boldsymbol{1}\{\cdots\} = 1$, and $\mathbb{P}(\emptyset) = 1$; we will show shortly that the denominator in \eqref{corollary1.5} is strictly positive, so the right-hand side is well defined.

 In addition, in light of the assumption that the joint density of $\boldsymbol{X}_J$ is uniformly bounded above and below by positive constants for every subset $J\subset \{1, \dots, p\}$ with $\texttt{\#}J \le s_0 + 1$, and the definitions of $L_{1n}$ and $L_{2n}$, we deduce that there are some positive constants $C_{1}$ and $C_{2}$ such that for every $1\le l_j < m_j$, every $n\ge 1$, and every $J\subset \{1, \dots, p\}$ with $1\le \texttt{\#}J \le 1+ s_0$,
 \begin{equation}
      \label{corollary1.1}
  (L_{1n})^{\texttt{\#}J} C_1 \le \mathbb{E}\left(\prod_{j\in J}\boldsymbol{H}_{jl_j}\right) \le (L_{2n})^{\texttt{\#}J} C_2.
  \end{equation}

By \eqref{corollary1.5}--\eqref{corollary1.1}, and the assumptions that $\texttt{\#}S_j\le s_0$, $L_{2n}\ge L_{1n}$, and $\sup_{n\ge 1} L_{2n}L_{1n}^{-1}  $ is finite, there exists a generic constant $C >0$ such that for every $n\ge 1$,
\begin{equation}    
    \begin{split}\label{corollary1.6}
    \min_{j, l < m_j, \omega_q(\vv{h}_{q}) \not = m_q \text{ for } 1\le q< j } e_{jl}(\vv{h}_{j-1}, \dots, \vv{h}_{1} )  & \ge \frac{(L_{1n})^{s_0 + 1} C_1}{(L_{2n})^{s_0} C_2}\\
   & \ge L_{1n} C \\
   & \eqqcolon d_{\min}.
    \end{split}
\end{equation}

We now turn to establishing the high probability deviation upper bound as required by the remaining part of Condition~\ref{mean.regression}. Analogously to \eqref{corollary1.6}, for some generic constant \( C > 0 \), it holds for every $n\ge1$ that
\begin{equation}
    \label{corollary1.13}
    \begin{split}        
   \max_{j, l < m_j, \omega_q(\vv{h}_{q}) \not = m_q \text{ for } 1\le q< j } e_{jl}(\vv{h}_{j-1}, \dots, \vv{h}_{1} ) & \le 
    \frac{(L_{2n})^{s_0 + 1} C_2}{(L_{1n})^{s_0} C_1}\\
   & \le L_{2n} C.
    \end{split}
\end{equation}
Also, since $\boldsymbol{H}_{jl}\in\{0,1\}$ almost surely, a direct calculation shows that for every $\iota\ge 0$, $J \subset \{1, \dots, p\}$, and $1\le l_{j}\le m_j$,
  \begin{equation}
      \label{corollary1.2}
  \mathbb{E}\left|\prod_{j\in J}\boldsymbol{H}_{jl_j} - \mathbb{E}\left( \prod_{j\in J}\boldsymbol{H}_{jl_j} \right) \right|^{2+\iota} \le  \mathbb{E}\left( \prod_{j\in J}\boldsymbol{H}_{jl_j} \right) .
\end{equation}
Here, we recall that $\boldsymbol{H}_{1jl}, \dots, \boldsymbol{H}_{njl}$ are identically distributed as $\boldsymbol{H}_{jl}$.
We also define sample one-hot $\vv{\boldsymbol{H}}_{tj} = (\boldsymbol{H}_{tj1}, \dots, \boldsymbol{H}_{tjm_j})^{\top}$, 
whose stationary population counterpart is denoted by $\vv{\boldsymbol{H}}_{j} = (\boldsymbol{H}_{j1}, \dots, \boldsymbol{H}_{jm_j})^{\top}$.
Let any $j\in \{1, \dots, p\}$, $l\in \{1, \dots, m_j -1\}$, and one-hot $(\vv{h}_{1}, \dots, \vv{h}_{p})$ with $\omega_q(\vv{h}_q) < m_q$ be fixed, and define
\begin{equation}
    \label{corollary1.21}
    \boldsymbol{Z}_t = \sqrt{\frac{\max_{1\le q\le p} m_{q}}{\mathbb{E}(\prod_{ q\in S_{j}}\boldsymbol{1}\{\vv{\boldsymbol{H}}_{q} = \vv{h}_q \}) }} \times \boldsymbol{H}_{tjl} \times \prod_{ q\in S_{j}}\boldsymbol{1}\{\vv{\boldsymbol{H}}_{tq} = \vv{h}_q \}.
\end{equation}
By \eqref{corollary1.1} and \eqref{corollary1.2},  it holds that for any $0<\iota\le 1$,
\begin{equation}
    \label{corollary1.7}
    \mathbb{E}| \boldsymbol{Z}_t - \mathbb{E}(\boldsymbol{Z}_t )|^{2+\iota} \le (\max_{1\le q\le p} m_{q})^{1+\iota/2}(L_{1n})^{-\texttt{\#}S_{j} - \frac{\texttt{\#}S_{j} \times \iota }{2} } C_{1}^{-1-\iota/2} (L_{2n})^{\texttt{\#}S_{j} + 1} C_2 \le C_5 L_{1n}^{- \frac{1}{2}(s_0+1) \iota } 
\end{equation}
for some generic constant $C_5>0$, since $\max_{1\le j\le p} m_j \le L_{1n}^{-1} C$ for some constant $C>0$ and  $\sup_{n\ge 1} L_{2n}L_{1n}^{-1}  $ is finite. By \eqref{corollary1.7}, the geometric strong mixing property of $\{\boldsymbol{Z}_t\}$ with 
$\alpha(k) \le \gamma_0 e^{-\gamma_1 k}$, 
and the standard covariance inequality for $\alpha$-mixing sequences~\citep{rio1993covariance}, 
we deduce that
\begin{equation}
    \label{corollary1.8}
     2\sum_{k=0}^{\infty} |\mathrm{Cov}(\boldsymbol{Z}_0, \boldsymbol{Z}_{k})| \le C\times L_{1n}^{\frac{-\iota(s_0+1)}{2+\iota}} \eqqcolon v^2
\end{equation}
for some generic constant $C>0$. Here, the mixing property of $\{\boldsymbol{Z}_t\}$ holds because $\boldsymbol{Z}_t$ is a measurable function of $(\boldsymbol{X}_{il}, l\in S_j\cup\{j\})$ 
and $\{(\boldsymbol{X}_{il}, l\in S_j\cup\{j\})\}_i$ is strongly mixing with coefficient $\alpha(k)$ according to Condition~\ref{alpha}.

By Theorem 2 of \citep{merlevede2009bernstein}, \eqref{corollary1.8}, that $\{\boldsymbol{Z}_t \}$ is geometrically strongly mixing with $\alpha(k) \le \gamma_0 e^{-\gamma_1 k}$ for all $k\ge 1$, and that
$$|\boldsymbol{Z}_t - \mathbb{E}(\boldsymbol{Z}_t)| \le \sqrt{\frac{\max_{1\le j\le p} m_{j}}{\mathbb{E}(\prod_{ q\in S_{j}}\boldsymbol{1}\{\vv{\boldsymbol{H}}_{q} = \vv{h}_q \}) }} \le C L_{1n}^{-1/2 - \texttt{\#}S_j / 2} \eqqcolon Q_n$$
for some generic constant $C>0$ almost surely (derived by using similar calculations for \eqref{corollary1.7}), there exists some $C_6 >0$ such that for all $n\ge 2$, 
\begin{equation}
    \label{corollary1.9}
    \mathbb{P}\left( \left|  \sum_{t=1}^{n} (\boldsymbol{Z}_t - \mathbb{E}(\boldsymbol{Z}_t))\right| \ge x \right) \le \exp{\left(\frac{C_6 x^2}{v^2 n + Q_n^2 + xQ_n (\log{n})^2} \right)}.
\end{equation}

By \eqref{corollary1.9}, setting $x = v \sqrt{n} (\log n)^4$ and using the assumption 
$L_{1n}^{-1 - \texttt{\#}S_j} \le L_{1n}^{-1 - s_0} \le n \le v^2 n$ for all large $n$ (the last inequality follows since $L_{1n}\le 1$ for all large $n$), 
it follows that, for some constant $C > 0$ and all sufficiently large $n$,
\begin{equation}
    \label{corollary1.10}
    \mathbb{P}\left( \left|  n^{-1}\sum_{t=1}^{n} (\boldsymbol{Z}_t - \mathbb{E}(\boldsymbol{Z}_t))\right| \ge  v  n^{-1/2} (\log{n})^4 \right) \le \exp{\left(-C(\log{n})^2 \right)}.
\end{equation}
By letting $\boldsymbol{G}_t = \sqrt{\frac{1}{\mathbb{P}(\cap_{ q\in S_{j}}\{\vv{\boldsymbol{H}}_{q} = \vv{h}_q \}) }}  \times \prod_{ q\in S_{j}}\boldsymbol{1}\{\vv{\boldsymbol{H}}_{tq} = \vv{h}_q \}$ and a similar argument for deriving \eqref{corollary1.10}, we obtain some generic constant $C^{'}>0$ such that for all large $n$,
\begin{equation}
    \label{corollary1.15}
    \mathbb{P}\left( \left|  n^{-1}\sum_{t=1}^{n} (\boldsymbol{G}_t - \mathbb{E}(\boldsymbol{G}_t))\right| \ge  v  n^{-1/2} (\log{n})^4 \right) \le \exp{\left(-C^{'} (\log{n})^2 \right)}.
\end{equation}

Recall the definition of $\widehat{e}_{jl}(\vv{h}_{j-1}, \dots, \vv{h}_{1})$ given in \eqref{tree}. In \eqref{corollary1.12}--\eqref{corollary1.16} below, we establish a bound on 
\(| e_{jl}(\vv{h}_{j-1}, \dots, \vv{h}_{1}) - \widehat{e}_{jl}(\vv{h}_{j-1}, \dots, \vv{h}_{1}) |\),
under the assumption that the complementary events of the left-hand sides of
\eqref{corollary1.10} and \eqref{corollary1.15} hold.
\begin{equation}
\begin{split}\label{corollary1.11}
& | e_{jl}(\vv{h}_{j-1}, \dots, \vv{h}_{1} ) - \widehat{e}_{jl}(\vv{h}_{j-1}, \dots, \vv{h}_{1} )|    \\
& \le \left|\frac{\mathbb{E}[\boldsymbol{H}_{jl} \times \prod_{ q\in S_{j}}\boldsymbol{1}\{\vv{\boldsymbol{H}}_{q} = \vv{h}_q \} ]}{\mathbb{P}(\cap_{ q\in S_{j}}\{\vv{\boldsymbol{H}}_{q} = \vv{h}_q \}) } - \frac{ n^{-1} \sum_{t =1 }^n \boldsymbol{H}_{tjl} \times \prod_{ q\in S_{j}}\boldsymbol{1}\{\vv{\boldsymbol{H}}_{tq} = \vv{h}_q \}}{ n^{-1}\sum_{t =1 }^n \prod_{ q\in S_{j}}\boldsymbol{1}\{\vv{\boldsymbol{H}}_{tq} = \vv{h}_q\}  } \right| \\
& \le \left|\frac{\mathbb{E}[\boldsymbol{H}_{jl} \times \prod_{ q\in S_{j}}\boldsymbol{1}\{\vv{\boldsymbol{H}}_{q} = \vv{h}_q \} ] -  n^{-1} \sum_{t =1 }^n \boldsymbol{H}_{tjl} \times \prod_{ q\in S_{j}}\boldsymbol{1}\{\vv{\boldsymbol{H}}_{tq} = \vv{h}_q \}}{\mathbb{P}(\cap_{ q\in S_{j}}\{\vv{\boldsymbol{H}}_{q} = \vv{h}_q \}) }  \right| \\
& \quad + \widehat{e}_{jl}(\vv{h}_{j-1}, \dots, \vv{h}_{1} ) \times  \left|\frac{ \mathbb{P}(\cap_{ q\in S_{j}}\{\vv{\boldsymbol{H}}_{q} = \vv{h}_q \}) -  n^{-1} \sum_{t =1 }^n \prod_{ q\in S_{j}}\boldsymbol{1}\{\vv{\boldsymbol{H}}_{tq} = \vv{h}_q \}}{\mathbb{P}(\cap_{ q\in S_{j}}\{\vv{\boldsymbol{H}}_{q} = \vv{h}_q \}) }  \right|.
\end{split}
\end{equation}
When the complementary event of the left-hand sides of
\eqref{corollary1.10} holds, the first term on the RHS of \eqref{corollary1.11} is bounded by 
\begin{equation}
    \begin{split}
        \label{corollary1.12}
        v  n^{-1/2} (\log{n})^4 \sqrt{\frac{1}{ \mathbb{P}(\cap_{ q\in S_{j}}\{\vv{\boldsymbol{H}}_{q} = \vv{h}_q \}) \times \max_{1\le j\le p} m_j } }\le C^{'} v  n^{-1/2} (\log{n})^4 L_{2n}^{\frac{1 - \texttt{\#}S_j}{2}},
    \end{split}
\end{equation}
for some generic constant $C^{'}>0$ for all large $n$, where we use $\max_{1\le j\le p} m_j  \ge L_{2n}^{-1}C$ for some constant $C>0$, \eqref{corollary1.1}, and the assumption that $\sup_{n\ge 1} L_{2n}L_{1n}^{-1}  $ is finite. In addition, when the complementary events of the left-hand sides of \eqref{corollary1.10} and \eqref{corollary1.15} hold, by the recursive argument together with
$$\widehat{e}_{jl}(\vv{h}_{j-1}, \dots, \vv{h}_{1} ) \le |e_{jl}(\vv{h}_{j-1}, \dots, \vv{h}_{1} ) - \widehat{e}_{jl}(\vv{h}_{j-1}, \dots, \vv{h}_{1} )| + e_{jl}(\vv{h}_{j-1}, \dots, \vv{h}_{1} ),$$
that $0\le \widehat{e}_{jl}(\vv{h}_{j-1}, \dots, \vv{h}_{1} ) \le 1$, \eqref{corollary1.1}, \eqref{corollary1.13}, and the assumption that $\sup_{n\ge 1} L_{2n}L_{1n}^{-1}  $ is finite, the second term on the RHS of \eqref{corollary1.11} is bounded by 
\begin{equation}
    \begin{split}
        \label{corollary1.16}
        & C^{'} v  n^{-1/2} (\log{n})^4\times e_{jl}(\vv{h}_{j-1}, \dots, \vv{h}_{1} ) \sqrt{\frac{1}{ \mathbb{P}(\cap_{ q\in S_{j}}\{\vv{\boldsymbol{H}}_{q} = \vv{h}_q \})  }} \\
        & \le C^{''} v  n^{-1/2} (\log{n})^4\times L_{2n}\sqrt{\frac{1}{ \mathbb{P}(\cap_{ q\in S_{j}}\{\vv{\boldsymbol{H}}_{q} = \vv{h}_q \})  }} \\
        & \le C^{'''} v  n^{-1/2} (\log{n})^4 L_{2n}^{1 - \frac{ \texttt{\#}S_j}{2}}.
    \end{split}
\end{equation}
for positive generic constants $C^{'}, C^{''}, C^{'''}$ for all large $n$.

Now, let us deal with the number of deviation bounds needed. The transformed process $\{\boldsymbol{Z}_t\}$ given by \eqref{corollary1.21}
is indexed by $j\in \{1, \dots, p\}$, $l\in \{1, \dots, m_j -1\}$, and $(\vv{h}_q, q\in S_{j})$. As a result, the total number of means to be estimated is bounded by
\begin{equation}    
    \begin{split} \label{corollary1.17}
    2\times (m_1 + \sum_{j=2}^p \sum_{l=1}^{m_j} \prod_{q\in S_{j}} m_q ) & \le 2m_1 + 2p (\max_{1\le j\le p}m_j)^{1+s_0} \\
    & \le 4p (\max_{1\le j\le p}m_j)^{1+s_0} \\
    & \le 4n^{2+s_0},
    \end{split}
\end{equation}
since $p\le n$ (due to our assumption $pn^{-1/2} (\log{n})^5 L_{2n}^{-\frac{1}{2} - \frac{s_0}{2} - b_0}\le 1$) and $m_{j}\le n$. The same upper bound holds for the number of $\{\boldsymbol{G}_t\}$. Applying \eqref{corollary1.10}--\eqref{corollary1.15} and \eqref{corollary1.17}, we conclude that the union of the events corresponding to the left-hand sides of \eqref{corollary1.10}--\eqref{corollary1.15} has negligible probability as \(n\) increases; denote this probability by \(\Delta = \Delta_n\).

By this, \eqref{corollary1.8}, \eqref{corollary1.10}--\eqref{corollary1.16}, and the assumption that $L_{2n} = o(1)$, with
\begin{equation}
    \label{T.n}
    \sqrt{T}\coloneqq C L_{1n}^{\frac{-\iota(s_0+1)}{2+\iota}}  n^{-1/2} (\log{n})^4 L_{2n}^{\frac{1 - s_0}{2}}
\end{equation}
for some sufficiently large constant $C>0$, we have Condition~\ref{mean.regression} with $\Delta = \Delta_n$ approaching zero as $n$ increases. 

In addition, by the assumption that $\sup_{n\ge 1} L_{2n}L_{1n}^{-1}  $ is finite, and \eqref{corollary1.6}, it holds for all large $n$ that 
    \begin{equation}
    \label{corollary1.20}
    \frac{ep\sqrt{T}}{d_{\min}} \le pn^{-1/2} (\log{n})^5 L_{2n}^{-\frac{1}{2} - \frac{s_0}{2} - b_0}\le 1,
\end{equation}
where \(\iota \in (0,1]\) is arbitrary, $s_0$ is finite, and \(b_0 > 0\) is chosen such that \(b_0 = \frac{\iota\times (s_0+1)}{2 + \iota}\).

In summary, Condition~\ref{mean.regression} holds for all sufficiently large \(n\), with
\(d_{\min} = d_{\min,n}\) given by~\eqref{corollary1.6} and \(T = T_n\) given by~\eqref{T.n}. 
Since \(\varepsilon = \varepsilon_n = L_{2n}\sqrt{p}\) by definition of \(L_{2n}\), the assumption 
\(L_{2n} p \delta_n = o(1)\) implies \(\varepsilon \le (\delta_n \sqrt{p}\, K_{1})^{-1}\) for all large \(n\). 
Together with~\eqref{corollary1.20}, the assumption of Condition~\ref{conti} with \(R = \delta_n\), and an application of Theorem~\ref{theorem1}, we obtain the desired conclusion of Corollary~\ref{corollary1}.

\subsection{Proof of Corollary~\ref{corollary2}}\label{proof.corollary2}

We verify that Condition~\ref{mean.regression} holds for all large $n$ in the setting of Corollary~\ref{corollary2} and then apply Theorem~\ref{theorem1} to establish Corollary~\ref{corollary1}. We begin by deriving a positive constant lower bound \(d_{\min} > 0\) as required by Condition~\ref{mean.regression}.

By the definition of $e_{jl}$ in \eqref{band}, the discrete variables assumption, and Condition~\ref{sparsity}(a), it holds that
\begin{equation}
\begin{split}\label{corollary2.1}
    e_{jl}(\vv{h}_{j-1}, \dots, \vv{h}_{1} ) & = \mathbb{E}[ \boldsymbol{H}_{jl} \mid \vv{\boldsymbol{H}}_{j-1} = \vv{h}_{j-1}, \dots, \vv{\boldsymbol{H}}_{1} = \vv{h}_{1}]
    \\
    & = \mathbb{E}[ \boldsymbol{H}_{jl} \mid \vv{\boldsymbol{H}}_q = \vv{h}_q, q\in S_{j}]\\
    & = \frac{\mathbb{E}[\boldsymbol{H}_{jl} \times \prod_{ q\in S_{j}}\boldsymbol{1}\{\vv{\boldsymbol{H}}_{q} = \vv{h}_q \} ]}{\mathbb{P}(\cap_{ q\in S_{j}}\{\vv{\boldsymbol{H}}_{q} = \vv{h}_q \}) }.
    \end{split}
\end{equation}

By \eqref{corollary2.1} and the assumption of upper and lower bounds on $\mathbb{E}(\prod_{j\in S}\boldsymbol{H}_{jl_j})$, we deduce that
$$\min_{1\le j\le p; \, 1\le l < m_j ; \, \omega_q(\vv{h}_{q}) < m_q, 1\le q< j } e_{jl}(\vv{h}_{j-1}, \dots, \vv{h}_{1}) \ge d_{\min}$$ 
for some constant $d_{\min} > 0$, which ensures the first part of Condition~\ref{mean.regression}.

Next, we establish the remaining part of Condition~\ref{mean.regression} for applying Theorem~\ref{theorem1}.
Let any $j\in \{1, \dots, p\}$, $l\in \{1, \dots, m_j -1\}$, and one-hot $(\vv{h}_{1}, \dots, \vv{h}_{p})$ with $\omega_q(\vv{h}_q) < m_q$ be fixed, and define
\begin{equation}
    \label{corollary2.21}
    \boldsymbol{Z}_t = \sqrt{\frac{1}{\mathbb{E}(\prod_{ q\in S_{j}}\boldsymbol{1}\{\vv{\boldsymbol{H}}_{q} = \vv{h}_q \}) }} \times \boldsymbol{H}_{tjl} \times \prod_{ q\in S_{j}}\boldsymbol{1}\{\vv{\boldsymbol{H}}_{tq} = \vv{h}_q \}.
\end{equation}
The transformed process $\{\boldsymbol{Z}_t\}$ is indexed by $j \in \{1, \dots, p\}$, $l \in \{1, \dots, m_j-1\}$, and $(\vv{h}_q, q \in S_j)$. Consequently, the total number of distinct transformed processes is bounded as in \eqref{corollary1.17}.

By the assumption of upper and lower bounds on $\mathbb{E}(\prod_{j\in S}\boldsymbol{H}_{jl_j})$ with $\texttt{\#}S \le s_0 + 1$, \(\max_{1 \le j \le p} \texttt{\#}S_j \le s_0\), and that $s_0$ is a constant, it holds that for any $0<\iota\le 1$,
\begin{equation}
    \label{corollary2.7}
    \mathbb{E}| \boldsymbol{Z}_t - \mathbb{E}(\boldsymbol{Z}_t )|^{2+\iota} \le C
\end{equation}
for some generic constant $C>0$. Moreover, by an argument analogous to that used in \eqref{corollary1.8}, we obtain that, for Corollary~\ref{corollary2}, the upper bound $v^2$ on $2\sum_{k=0}^{\infty} |\mathrm{Cov}(\boldsymbol{Z}_0, \boldsymbol{Z}_{k})|$ is constant and that $|\boldsymbol{Z}_t - \mathbb{E}(\boldsymbol{Z}_t)|\le C \eqqcolon Q_n$ for some constant $ C > 0$. Setting $x = n^{-1/2} (\log n)^4$ and applying the same arguments as in \eqref{corollary1.9}--\eqref{corollary1.10} and \eqref{corollary1.11}--\eqref{corollary1.16} yields analogous results here (details omitted). Also, the results in \eqref{corollary1.15} carry over directly. Combining these with \eqref{corollary1.17} and $\sqrt{T} = \sqrt{T_n}\coloneqq C n^{-1/2} (\log n)^4$ for a sufficiently large constant $C>0$, we obtain Condition~\ref{mean.regression} with $\Delta = \Delta_n \to 0$ as $n \to \infty$ for Corollary~\ref{corollary2} for all large $n$; additionally, we have that $\frac{ep\sqrt{T}}{d_{\min}} \le pn^{-1/2} (\log{n})^5 \le 1$ for all large $n$ due to the assumption that $p n^{-1/2} (\log n)^5 \le 1$.

Applying Theorem~\ref{theorem1} with the above result, along with \(\mathbb{P}(\vv{\boldsymbol{X}} \notin [-\delta_n, \delta_n]) = 0\) and \(\varepsilon = 0\), yields the desired conclusion of Corollary~\ref{corollary2}. Note that Condition~\ref{conti} is satisfied here because we have defined $\pi_{S, \vv{d}_S}(\vv{x}) = 1$ on $\mathbb{R}^p$ when $\mathbb{R}^p(S, \vv{d}_S)$ containing a discrete point (see Section~\ref{Sec2.1}), and that $\Theta^{\star} = \Theta$ when $\vv{\boldsymbol{X}}$ is discrete (see Section~\ref{Sec4.2}).

We have completed the proof of Corollary~\ref{corollary2}.

\subsection{Lemma~\ref{lemma11} and its proof}\label{proof.lemma11}

\begin{lemma}\label{lemma11}
    Assume that $f(\vv{x})$ is positive and satisfies 
$\|\nabla \log f(\vv{x})\|_2 \le K_{1} \sqrt{p}\,R$ 
for all $\vv{x}$ such that $\|\vv{x}\|_{\infty} \le R$, 
where $K_{1} > 0$ is a constant. Then, for every $\vv{x}, \vv{y}\in \mathbb{R}^p$ with $\norm{\vv{x}}_{\infty}\le R$, $\norm{\vv{y}}_{\infty} \le R$, and $\norm{\vv{y} - \vv{x}}_2 \le \varepsilon$, it holds that
    $$
     |f(\vv{y}) - f(\vv{x})| 
    \le  K_{1} \varepsilon   \sqrt{p}R \times f(\vv{x}) \times \exp(K_{1}  \sqrt{p}R \varepsilon),
    $$
    
\end{lemma}

\noindent\textit{Proof of Lemma~\ref{lemma11}: } For any continuously differentiable function $f : \mathbb{R}^p \to \mathbb{R}$ and any $\vv{x}, \vv{y} \in \mathbb{R}^p$, the multivariate version of the Fundamental Theorem of Calculus (or gradient integral formula along a line) gives
\[
f(\vv{y}) - f(\vv{x})
= \int_0^1 \nabla f(\vv{x} + t(\vv{y}-\vv{x}))^{\top} (\vv{y}-\vv{x}) \, dt.
\]
Hence,
\[
|f(\vv{y}) - f(\vv{x})|
\le \|\vv{y}-\vv{x}\|_2
\sup_{t \in [0,1]} \|\nabla f(\vv{x} + t(\vv{y}-\vv{x}))\|_2,
\]
implying that for every $\|\vv{x}\|_{\infty} \le R$, $\|\vv{y}\|_{\infty} \le R$, $\|\vv{x} - \vv{y}\|_2 \le \varepsilon$,
\begin{equation}
    \label{fundamental.1}
     |f(\vv{y}) - f(\vv{x})|
\le \varepsilon \sup_{\|\vv{z}\|_{\infty} \le R, \|\vv{x} - \vv{z}\|_2 \le \varepsilon} \|\nabla f(\vv{z})\|_2.
\end{equation}

Additionally, we apply the fundamental theorem of calculus along the line segment connecting $\vv{x}$ and $\vv{y}$:
\[
\log f(\vv{y}) - \log f(\vv{x})
= \int_{0}^{1} 
\nabla \log f\big(\vv{x} + t(\vv{y} - \vv{x})\big)^{\top} 
(\vv{y} - \vv{x})\, dt.
\]
This, along with our assumption that 
$\|\nabla \log f(\vv{x})\|_2 \le K_{1} \sqrt{p}\,R$ 
for all $\|\vv{x}\|_{\infty} \le R$, 
implies that for any $\vv{x}, \vv{y}$ satisfying 
$\|\vv{x}\|_{\infty} \le R$, $\|\vv{y}\|_{\infty} \le R$, 
and $\|\vv{x} - \vv{y}\|_2 \le \varepsilon$, we have
\[
|\log f(\vv{y})-\log f(\vv{x})|
\le \varepsilon\sup_{\|\vv{z}\|_{\infty} \le R, \|\vv{x} - \vv{z}\|_2 \le \varepsilon}\|\nabla\log f(\vv{z})\|_2
\le \varepsilon\times K_{1}\sqrt{p}\,R.
\]
Hence, for every $\|\vv{x}\|_{\infty} \le R$, $\|\vv{y}\|_{\infty} \le R$, $\|\vv{x} - \vv{y}\|_2 \le \varepsilon$,
\begin{equation*}    
    \exp(\log f(\vv{y})-\log f(\vv{x})) =  \frac{f(\vv{y})}{f(\vv{x})} \le \exp(\varepsilon\times K_{1}\sqrt{p} R),
\end{equation*}
leading to that for every $\|\vv{x}\|_{\infty} \le R$,
\begin{equation}
    \label{lemma10.2}
    \sup_{ \|\vv{z}\|_{\infty}\le R, \|\vv{z} - \vv{x}\|_2 \le \varepsilon} f(\vv{z}) \le f(\vv{x}) \exp(\varepsilon\times K_{1}\sqrt{p} R).
\end{equation}

By \eqref{fundamental.1}--\eqref{lemma10.2}, and our assumption, we deduce that for every $\|\vv{x}\|_{\infty} \le R$, $\|\vv{y}\|_{\infty} \le R$, $\|\vv{x} - \vv{y}\|_2 \le \varepsilon$,
 \begin{equation*}
     \begin{split}
          |f(\vv{y}) - f(\vv{x})|
& \le \varepsilon  \sup_{ \|\vv{z}\|_{\infty}\le R, \|\vv{z} - \vv{x}\|_2 \le \varepsilon }\|\nabla f(\vv{z})\|_2 \\
& = \varepsilon \sup_{ \|\vv{z}\|_{\infty}\le R, \|\vv{z} - \vv{x}\|_2 \le \varepsilon } [\|\nabla \log f(\vv{z})\|_2 f(\vv{z})]\\
& \le \varepsilon \times [\sup_{ \|\vv{z}\|_{\infty}\le R  } \|\nabla \log f(\vv{z})\|_2 ] \times \sup_{ \|\vv{z}\|_{\infty}\le R, \|\vv{z} - \vv{x}\|_2 \le \varepsilon } f(\vv{z})\\
& \le  \varepsilon \times K_{1} \sqrt{p}R \times f(\vv{x}) \times \exp(K_{1}  \sqrt{p}R \varepsilon),
     \end{split}
 \end{equation*}
 where the first inequality follows from \eqref{fundamental.1}, the first equality is due to $f(\vv{x}) > 0$ on $\|\vv{x}\|_{\infty}\le R$, and the third inequality results from our assumptions.

	\renewcommand{\theequation}{D.\arabic{equation}}
	\renewcommand{\thesubsection}{D.\arabic{subsection}}
	\setcounter{equation}{0}

\section{Supplementary Material}

\subsection{Examples of Distributions Satisfying Condition~\ref{conti}}\label{example.distribution}

We present three examples of distributions satisfying Condition~\ref{conti}. To handle the non-differentiability of the uniform distribution at its boundary, we consider a slightly enlarged support in Example~\ref{unif} (or, alternatively, a smoothly clipped version). Note that finite mixtures of these distributions, such as Gaussian or Student-$t$ mixtures, similarly satisfy Condition~\ref{conti}, though we omit those derivations for brevity.

\begin{exmp}\label{studentT}
Let $f(\vv{x})$ be the multivariate centered Student-$t$ density with a uniformly positive definite scale matrix $\Sigma\in\mathbb{R}^{p\times p}$ and degrees of freedom $\nu>0$. Assume $p / \nu$ is uniformly bounded. Then, Condition~\ref{conti} holds for every $R>0$.

\end{exmp}

\begin{exmp}\label{mulgaussian}
Let $f(\vv{x})$ be the multivariate Gaussian distribution with a uniformly positive definite covariance matrix. Then, Condition~\ref{conti} holds for every $R>0$.
\end{exmp}

\begin{exmp}\label{unif}
Let $f(\vv{x})$ be a uniform distribution on $[-\delta_n -  \epsilon, \delta_n + \epsilon]^p$ for some small $\epsilon>0$. Then, Condition~\ref{conti} holds for $R= \delta_n$.
\end{exmp}

In Examples~\ref{studentT}--\ref{unif}, constant $K_{1}$ of Condition~\ref{conti} is independent of $R$.

\noindent\textit{Proof of Examples~\ref{studentT}--\ref{unif}: } We begin with establishing Condition~\ref{conti} for Example~\ref{studentT}. Let $f$ be the multivariate Student-$t$ density with location $\vv{\mu}\in\mathbb{R}^p$, positive definite scale matrix $\Sigma\in\mathbb{R}^{p\times p}$, and degrees of freedom $\nu>0$. Specifically,
\[
f(\vv{x}) =
\frac{
\Gamma\!\big(\frac{\nu + p}{2}\big)
}{
\Gamma\!\big(\frac{\nu}{2}\big) \, (\nu \pi)^{p/2} \, D(\Sigma)^{1/2}
}
\left(
1 + \frac{1}{\nu} (\vv{x}-\vv{\mu})^{\top} \Sigma^{-1} (\vv{x}-\vv{\mu})
\right)^{-(\nu + p)/2}, 
\]
where $D(\Sigma)$ denotes the determinant of the scale (or covariance) matrix. Additionally, we assume in Example~\ref{studentT} that $\vv{\mu}$ is a vector of zeros.

Let us begin the formal proof. The derivative of log-density is given by
\[
\nabla\log f(\vv{x})
= -\frac{\nu+p}{\nu}\times\frac{\Sigma^{-1}\vv{x}}{1+\vv{x}^{\top}\Sigma^{-1}\vv{x}/\nu}.
\]
Hence
\[
\|\nabla\log f(\vv{x})\|_2
\le \frac{\nu+p}{\nu}\,\frac{\|\Sigma^{-1}\|_2\,\|\vv{x}\|_2}{1+\vv{x}^{\top}\Sigma^{-1}\vv{x}/\nu}
\le \frac{\nu+p}{\nu}\,\|\Sigma^{-1}\|_2\,\|\vv{x}\|_2,
\]
since $(1+\vv{x}^{\top}\Sigma^{-1}\vv{x}/\nu)^{-1}\le1$. If $\max_j|x_j|\le R$ then
\(\|\vv{x}\|_2 \le \sqrt{p}\,R \), so
\begin{equation}
    \label{example.student.1}
    \|\nabla\log f(\vv{x})\|_2
\le \frac{\nu+p}{\nu}\,\|\Sigma^{-1}\|_2\sqrt{p}\,R.
\end{equation}
Thus one can take  \(K_{1}=(\nu+p)/\nu\;\|\Sigma^{-1}\|_2\) to conclude Condition~\ref{conti} for every $R>0$. We have completed the proof of Example~\ref{studentT}.

Now, we turn to establish Condition~\ref{conti} for Example~\ref{mulgaussian}. Let $f$ denote the multivariate Gaussian density:
\[
f(\vv{x}) =
\frac{1}{(2\pi)^{p/2} D(\Sigma)^{1/2}}
\exp\!\Big(-\tfrac{1}{2}(\vv{x}-\vv{\mu})^{\top}\Sigma^{-1}(\vv{x}-\vv{\mu})\Big),
\]
where \( \vv{\mu} \in \mathbb{R}^p \) is the mean vector, and
\( \Sigma \in \mathbb{R}^{p \times p} \) is a positive definite covariance matrix. In Example~\ref{mulgaussian}, we consider the scenario with $\vv{\mu}$ being a vector of zeros.

The log-density is given by
\[
\log f(\vv{x}) 
= -\frac{p}{2} \log(2\pi) 
- \frac{1}{2} \log D(\Sigma) 
- \frac{1}{2} \vv{x}^{\top} \Sigma^{-1} \vv{x}.
\]
Hence, the gradient of the log-density is
\[
\nabla \log f(\vv{x}) = -\Sigma^{-1} \vv{x}.
\]

Thus, for $\vv{x}\in \mathbb{R}^p$ with $\norm{\vv{x}}_{\infty}\le R$,
\[
\|\nabla \log f(\vv{x})\|_2 = \|\Sigma^{-1} \vv{x}\|_2 \le \|\Sigma^{-1} \|_2 \|\vv{x}\|_2\le\|\Sigma^{-1} \|_2 \sqrt{p}R.
\]

Using this result and setting $K_{1} = \|\Sigma^{-1}\|_2$, we establish Condition~\ref{conti} for all $R > 0$ in Example~\ref{mulgaussian}.

Meanwhile, the analysis for Example~\ref{unif} is straightforward and therefore omitted. This completes the proofs for Examples~\ref{studentT}--\ref{unif}.

\subsection{Proof for Section~\ref{Sec4.1}}\label{proof.sec4.1}

    When $r=0$, Condition~\ref{sparsity}(a) is satisfied under Examples~\ref{k_modals}--\ref{k_modals_unif} because $\boldsymbol{X}_1, \dots, \boldsymbol{X}_p$ are i.i.d. 
    
    On the other hand, consider Example~\ref{k_modals_unif} with $r>3$. If we set $\delta_n \ge \frac{3}{2} + r$ and choose sufficiently many evenly spaced splitting points, then for each variable $j\in \{1, \dots, p\}$, every one-dimensional bin interval $B \subset \mathbb{R}$ satisfies either $\sup B \le 1.5$ or $\inf B \ge 1.5$ (bins are mutually disjoint). Hence each bin intersects with at most one of the two modes, which implies that the mode of the $j$th variable is already determined once $\vv{\boldsymbol{H}}_{l}$ is observed for any $1 \le l < j$ in this scenario. Therefore, for every $j>1$, each $1\le k \le m_j$, and each $l\in \{j-1, \dots, 1\}$,
\[
\mathbb{P}\!\left(\boldsymbol{H}_{jk} = 1 \,\middle|\, \vv{\boldsymbol{H}}_{j-1}, \dots, \vv{\boldsymbol{H}}_{1}\right)
=
\mathbb{P}\!\left(\boldsymbol{H}_{jk} = 1 \,\middle|\, \vv{\boldsymbol{H}}_{l}\right),
\]
and Condition~\ref{sparsity}(a) holds in this case with $S_{j} = \{l\}$ for any $1\le l < j$.

\subsection{Strong Mixing Processes} \label{SecB.1}
A strong mixing or $\alpha$-mixing process is defined as follows.
A stochastic process $\{\vv{\boldsymbol{X}}_t\}_{t \ge 1}$ is said to be \textit{strongly mixing} (or $\alpha$-mixing) if there exists a sequence $\{\alpha(k)\}_{k \ge 1}$ with $\alpha(k) \to 0$ as $k \to \infty$ such that
\[
\alpha(k) \coloneqq \sup_{t \ge 1} \sup_{A \in \mathcal{F}_1^t, \, B \in \mathcal{F}_{t+k}^{\infty}} \big| \mathbb{P}(A \cap B) - \mathbb{P}(A)\mathbb{P}(B) \big|,
\]
where $\mathcal{F}_1^t$ and $\mathcal{F}_{t+k}^{\infty}$ are the $\sigma$-fields generated by $\{\vv{\boldsymbol{X}}_s : 1 \le s \le t\}$ and $\{\vv{\boldsymbol{X}}_s : s \ge t+k\}$, respectively. If there exist constants $\gamma_0, \gamma_1 > 0$ such that 
\[
\alpha(k) \le \gamma_0 e^{-\gamma_1 k} \quad \forall k \ge 1,
\] 
the process is called geometrically strongly mixing.


\end{document}